\documentclass[lettersize,journal,compsoc]{IEEEtran}
\usepackage{amsmath,amsfonts}
\usepackage{algorithmic}
\usepackage{array}
\usepackage[caption=false,font=normalsize,labelfont=sf,textfont=sf]{subfig}
\usepackage{textcomp}
\usepackage{stfloats}
\usepackage{url}
\usepackage{verbatim}
\usepackage{graphicx}

\usepackage{amssymb}
\usepackage{mathtools}
\usepackage{amsthm}

\usepackage{booktabs} %
\usepackage{bm}
\usepackage{multirow}
\usepackage{hyperref}

\theoremstyle{plain}
\newtheorem{theorem}{Theorem}[section]

\newtheorem{lemma}[theorem]{Lemma}
\newtheorem{corollary}[theorem]{Corollary}
\theoremstyle{definition}
\newtheorem{definition}[theorem]{Definition}

\theoremstyle{remark}

\newtheorem*{lemma*}{Lemma}  %
\newtheorem*{theorem*}{Theorem}

\def\BibTeX{{\rm B\kern-.05em{\sc i\kern-.025em b}\kern-.08em
    T\kern-.1667em\lower.7ex\hbox{E}\kern-.125emX}}
\usepackage{balance}
\begin{document}
\title{Stability-Shaped Deep Graph Learning}
\author{Junyou~Zhu,
        Langzhou~He,
        Fenying~Cai,
        Christian~Nauck,
        Ping~Xiong,
        Chao~Gao,
        Philip~S.~Yu,~\IEEEmembership{Life~Fellow,~IEEE},\\
        Klaus-Robert~Müller,
        Jürgen~Kurths,
        Frank~Hellmann

\thanks{Junyou~Zhu is with the Department of Complexity Science, Potsdam Institute for Climate Impact Research, 14473 Potsdam, Germany, and also with the Machine Learning Group, Technical University of Berlin, 10587 Berlin, Germany. Email: \texttt{Junyou.Zhu@pik-potsdam.de}}%

\thanks{Langzhou~He and Philip~S.~Yu are with the Department of Computer Science, University of Illinois at Chicago, Chicago, IL 60607, USA. Emails: \texttt{lhe24@uic.edu}, \texttt{psyu@uic.edu}}
\thanks{Fenying~Cai, Christian~Nauck, and Frank~Hellmann are with the Department of Complexity Science, Potsdam Institute for Climate Impact Research, 14473 Potsdam, Germany. Emails: \texttt{fenyingc@pik-potsdam.de}, \texttt{nauck@pik-potsdam.de}, \texttt{hellmann@pik-potsdam.de}}%
\thanks{Ping~Xiong is with the Berlin Institute for the Foundations of Learning and Data (BIFOLD), 10587 Berlin, Germany, also with the Machine Learning Group, Technical University of Berlin, 10587 Berlin, Germany. Email: \texttt{p.xiong@tu-berlin.de}}
\thanks{Chao Gao is with the School of Artificial Intelligence, Optics and Electronics (iOPEN), Northwestern Polytechnical University, Xi'an 710072, China. Email: \texttt{cgao@nwpu.edu.cn}}
\thanks{Jürgen Kurths is with the Research Institute of Intelligent Complex Systems, Fudan University, Shanghai 200433, China, also with the Department of Complexity Science, Potsdam Institute for Climate Impact Research, 14473 Potsdam, Germany, and also with the Department of Physics, Humboldt University Berlin, 12489 Berlin, Germany. Email: \texttt{kurths@pik-potsdam.de}.}
\thanks{Klaus-Robert Müller is with the Machine Learning Group, Technical University of Berlin, 10587 Berlin, Germany, also with the Berlin Institute for the Foundations of Learning and Data (BIFOLD), 10587 Berlin, Germany, also with the Department of Artificial Intelligence, Korea University, Seoul 136-713, South Korea, and also with the Max Planck Institute for Informatics, 66123 Saarbrücken, Germany. Email: \texttt{klaus-robert.mueller@tu-berlin.de}}
}

\maketitle

\begin{abstract}
In deep graph neural networks, increasing depth enlarges the receptive field but often leads to over-smoothing, where node representations tend to align. We develop a unified, mode-wise stability framework for deep GNN propagation that provides a principled characterization of over-smoothing.
By interpreting layer depth as time and layer updates as graph-coupled dynamics, over-smoothing can be understood as an undesirable dynamical synchronization of features, for which the master stability curve provides a theoretical tool to assess the stability of synchrony. Guided by this theory, we further propose \textbf{S}tability-\textbf{S}haped \textbf{D}eep \textbf{G}raph \textbf{L}earning (SDGL) to mitigate over-smoothing in deep GNNs. SDGL has two complementary instantiations: one induces controlled Turing instability to replace synchronization with spatial pattern formation, and the other maintains stable near-critical propagation. Experiments on diverse node- and graph-level benchmarks demonstrate the improved depth scaling and consistent accuracy gains over strong baselines, including graphs exhibiting long-range dependencies.
\end{abstract}

\begin{IEEEkeywords}
Graph neural networks, Deep graph learning, Over-smoothing, Spectral analysis, Stability analysis.
\end{IEEEkeywords}

\section{Introduction}
Increasing depth expands the receptive field of graph neural networks (GNNs) and is a natural route to long-range reasoning~\cite{chen2023bag,eliasoftango,chen2023bagtricks}.
Yet in many popular designs, deeper propagation makes node representations increasingly alike and can erase node-level distinctions~\cite{nguyen2023revisiting,sun2026asil,yuehyperbolic,yang2026deeper}.
This phenomenon is commonly referred to as \emph{over-smoothing}, where after many layers embeddings become nearly constant across nodes and the model loses discriminative power~\cite{thorpe2022grand11,attali2024delaunay,nauck2024dirac,park2024mitigating,li2023deepergcn}.
\begin{figure}[t] 
\vskip -0.0in
\begin{center}
\includegraphics[width=1.0\linewidth]{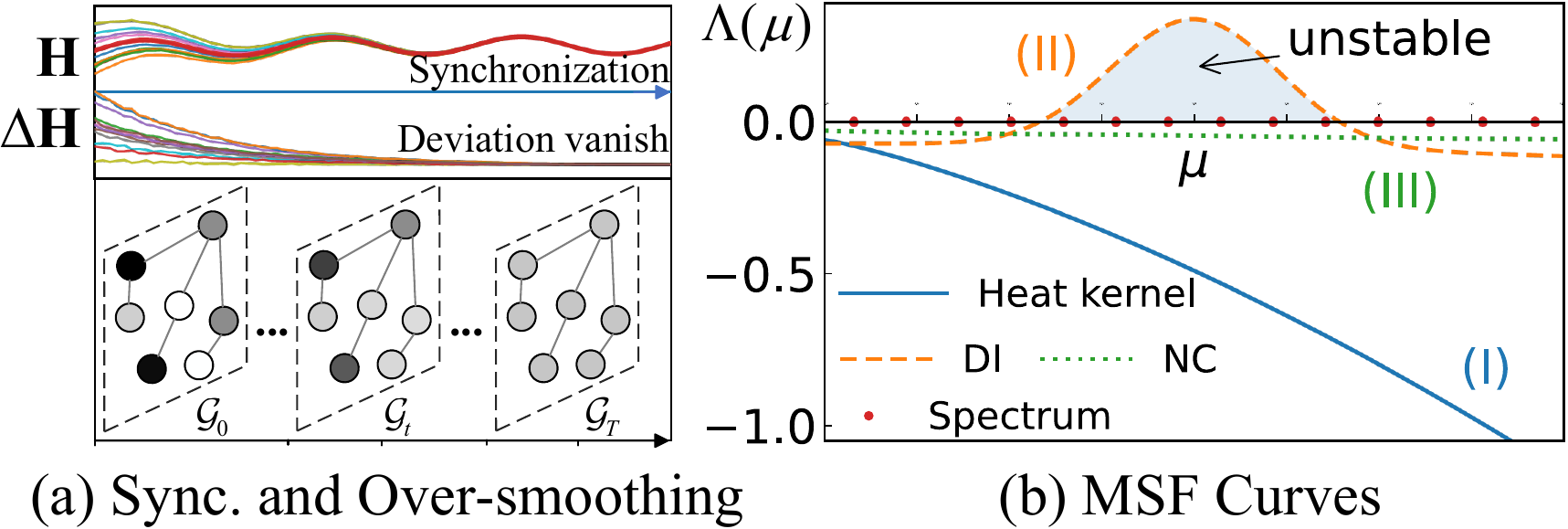}
\caption{
\textbf{(a)} Over-smoothing as synchronization: node representations $\mathbf H$ collapse toward a shared state and deviations from synchrony shrink in the middle panel.
\textbf{(b)} Our master-stability view: the curve $\Lambda(\mu)$ shows how each graph frequency $\mu$ is damped or amplified.
Heat-kernel diffusion is strongly contractive at high frequencies (I), our SDGL-DI opens a band-limited amplification window (II), and SDGL-NC remains stable but near-critical (III).
}

\label{fig_intro_syn_overs}
\end{center}
\vskip -0.0in
\end{figure}

For linear message passing, the mechanism behind over-smoothing is explicit. There, the propagation behaves like heat diffusion on a graph~\cite{zhao2025understanding,maskey2023fractional}, which progressively removes high-frequency variation and drives the signal toward an almost-constant configuration. From a dynamical-systems perspective, this almost-constant configuration corresponds to a \emph{synchronized} state in which all nodes share approximately the same representation. This perspective motivates us to investigate over-smoothing by analyzing whether small node-to-node differences shrink under repeated propagation.

A growing line of work views deep GNN propagation through the lens of dynamical systems~\cite{rusch2022graph,chamberlain2021grand,nguyen2024coupled,eliasof2024graph,ding2025let}.
Under this view, stacking layers corresponds to iterating a graph-coupled update, where node-wise nonlinearities interact with message passing over the graph~\cite{chamberlain2021grand}.
This perspective provides a clear picture of depth-induced information loss by offering a principled lens to track how node representations evolve under repeated propagation.
However, existing analyses are often tied to a particular architecture or discretization~\cite{rusch2022graph,chamberlain2021grand,eliasof2024graph}, making it hard to transfer conclusions across models and graphs. 

Motivated by this, we study over-smoothing as a stability phenomenon around the synchronized state.
The key point is that over-smoothing is not about the absolute value of embeddings, but about the \emph{decay of non-synchronous deviations} (node-to-node differences) under repeated propagation.
If the synchronized state is strongly attracting, any non-synchronous component is rapidly suppressed, leading to collapse (over-smoothing).
If it is weakly attracting, useful variation can persist at large depth. If it is repelling, deviations will grow exponentially, either without bound or until they are stabilized by nonlinear behavior far from the synchronous state.
Fig.~\ref{fig_intro_syn_overs}(a) illustrates this intuition.
This reframes the depth problem as a stability question: \textit{how fast do deviations from synchrony contract under repeated propagation, and how does this contraction rate determine the practically usable depth before informative variation is lost?}

The key challenge is to understand which graph-frequency components of node-to-node differences, equivalently which spectral modes of deviations from synchrony, decay fastest as depth increases.
For linear message passing, this analysis is straightforward.
The propagation operator has an eigenbasis, so any signal can be decomposed into independent spectral modes, and each mode decays at its own rate~\cite{keriven2022not}.
For nonlinear graph-coupled dynamics, however, the situation becomes more subtle.
The node-wise nonlinear update mixes feature channels and interacts with the graph coupling, so the spectral components of the differences are no longer obviously independent. Consequently, it becomes unclear which parts will be suppressed or amplified by deep propagation.

To address this, we develop a decoupled, spectral mode-wise stability theoretical framework inspired by the master stability formalism from dynamical systems~\cite{pecora1998master}.
The core idea is to linearize the dynamics around the synchronized state and express deviations in the eigenbasis of the graph operator.
This reduces the full coupled nonlinear system to a collection of linear systems parameterized by the graph eigenvalue.
As a result, each graph mode is associated with its own stability margin. Treating the graph eigenvalue as a free parameter, the stability margin provides a single stability curve that is independent of the specific graph.
This curve separates the contribution of graph topology from that of the local nonlinear update, and it predicts whether non-synchronous structure at a given frequency will be rapidly damped, selectively amplified, or only weakly contracted (Fig.~\ref{fig_intro_syn_overs}(b)).
Importantly, the curve directly suggests how to design architectures that either create a band-limited amplification window for pattern formation, or keep all non-synchronous modes stable but near-critical, so that useful variation persists at large depth.

\textbf{Our contributions can be summarized as follows:}
\begin{itemize} 
\item To the best of our knowledge, we develop the first unified master stability-inspired theoretical framework for nonlinear dynamical system-based GNNs that casts over-smoothing as a mode-wise stability problem in deep propagation.
\item Building on this theory, we propose two deep architectures that mitigate over-smoothing in complementary ways:
one selectively amplifies a band of graph frequencies, while the other keeps non-synchronous components stable and weakly damped for deep propagation. 
\item We validate our approach on homophilic and heterophilic node classification and on graph classification with long-range dependencies, and observe consistent gains from stability-shaped deep propagation.
\end{itemize}

\section{Related Work}

\subsection{Deep GNNs and depth-related degradation.}
While graph neural networks excel at learning representations by fusing local structure and node features~\cite{besta2026demystifying,shen2025torsion,zhu2022community, liu2026dictionary, zhu2025network}, standard shallow message passing struggles to capture long-range dependencies due to its restricted receptive field. Scaling to deeper architectures naturally addresses this limitation by enabling information flow across distant nodes~\cite{yang2026deeper, liu2025deep}. However, this depth introduces a severe trade-off: repeated neighborhood aggregation progressively dilutes node-specific variations, leading to over-smoothing and a critical loss of discriminative power.

To counter this depth-induced degradation, significant effort has been invested in optimizing the training and architectural design of deep GNNs. For instance, extensive empirical benchmarks reveal that carefully orchestrated training techniques can partially recover the performance of deep message-passing networks~\cite{chen2023bagtricks}. Other architectural interventions, such as the learnable generalized aggregation in \emph{DeeperGCN}~\cite{li2023deepergcn} or decoupled greedy training frameworks~\cite{zhang2024decouple}, explicitly aim to mitigate optimization hurdles and improve depth scalability. Yet, these practical advances collectively point to a more fundamental issue: the bottleneck in deep graph learning extends beyond empirical optimization. It is intrinsically tied to the long-horizon asymptotic behavior of the propagation operator itself. This realization directly motivates the need for theoretical tools capable of explicitly characterizing how information evolves and persists at large depths.

\subsection{Spectral filtering and mode-aware control.}
A complementary perspective interprets graph propagation as a form of spectral filtering~\cite{Zhu2025SDMG}, where repeated message passing typically acts as a low-pass operator. This inherent bias suppresses high-frequency variations, heavily favoring globally smooth representations. To counteract this and adapt to structural heterogeneity, a rich literature has emerged around designing flexible graph filters. For instance, learning node-dependent filters can alleviate the mismatch between global low-pass operations and local heterogeneous homophily conditions~\cite{zheng2024nodeoriented}. Crucially, these developments highlight that the impact of deep propagation is governed not just by the number of layers, but by precisely which graph frequencies are amplified or attenuated.

Parallel to modifying the filter response, other approaches aim to improve representation quality by fundamentally enriching the structural information available to the model. This includes augmenting standard neighborhood aggregation with subgraph isomorphism counting to strictly increase expressive power~\cite{bouritsas2023subgraph}, or extending graph learning to simplicial complexes to capture both direct and indirect higher-order relations~\cite{wu2024simplicial}.

Generally, these advancements highlight that scaling deep graph models demands a delicate balance between expressive structural encoding and controlled propagation behavior. While our work is rooted in this spectral viewpoint, we depart from the common objective of simply designing another expressive filter. Instead, we investigate a more fundamental mechanism: how each graph frequency mode inherently evolves, persists, or vanishes as depth increases. By introducing a graph-frequency-dependent stability curve, we explicitly characterize the long-term fate of deviations from synchrony. This reframing elevates mode-wise propagation behavior from an implicit consequence of a chosen filter into an explicitly analyzable and tunable object, forging a principled link between spectral structure, dynamical stability, and deep information preservation.

\subsection{Continuous dynamics and dynamical systems perspectives.}
Another compelling alternative to discrete message passing recasts graph propagation as the evolution of a continuous dynamical system, where layer depth naturally corresponds to time~\cite{rusch2022graph,gao2026hgnnv2}. This continuous perspective gives a rigorous mathematical way for analyzing representation collapse. For instance, framing deep graph learning through the lens of coupled oscillators demonstrates that mitigating over-smoothing is intrinsically linked to the stability of the underlying continuous ODEs~\cite{rusch2022graph}. Expanding on this, a broader class of continuous mechanisms, encompassing diffusion, wave, and reaction equations, provides a principled arsenal to counter the standard smoothing-dominated paradigm~\cite{han2024neuraldiffusion}. Theoretical investigations further corroborate this view: analyses of deep GCNs via Gaussian processes reveals that over-smoothing is not an inescapable fate, but rather depends on critical phase transitions within the limiting network dynamics~\cite{epping2024notoversmooth}.

Although these works establish the value of dynamical systems theory for understanding deep GNNs, many existing analyses remain tied to a particular architecture, discretization, or propagation rule. Our work is inspired by the classical master stability formalism for synchronized coupled systems~\cite{pecora1998master}, yet its goal is different from merely reinterpreting a given GNN through dynamical language. We develop a decoupled mode-wise description that separates graph topology from local update dynamics through a stability curve over graph frequencies. This decoupling is important because it elevates stability from an architecture-specific observation to a transferable design principle, enabling stability prediction and stability shaping across different graphs and deep propagation operators.

\subsection{Reaction diffusion and graph pattern formation.}
PDE-inspired graph models enrich the dynamical perspective by coupling simple diffusion with counteracting reaction or transport mechanisms~\cite{chamberlain2021grand}. This class of models is particularly well-suited for scaling network depth, as it introduces explicit physical priors to prevent indiscriminate smoothing. For instance, incorporating flexible reaction terms can actively sustain spatial variation, improving robustness against over-smoothing across diverse homophily conditions~\cite{choi2023gread}. Similarly, introducing directional transport alongside nonlinear channel interactions broadens the class of achievable propagation behaviors~\cite{eliasof2024featuretransport}. Beyond pure PDE analogies, disentangling the propagation flow over relevant versus irrelevant connections, such as through edge splitting, further reinforces the consensus that naive smoothing indiscriminately mixes both harmful and useful signals~\cite{guo2024esgnn}.

Our work is closest in spirit to these dynamical and PDE-motivated approaches, since we also aim to control how deep propagation acts on nontrivial graph structure. The key distinction, however, is that prior models typically obtain their stability behavior implicitly from a chosen transport, diffusion, or reaction mechanism, whereas we make the mode-wise stability profile itself an explicit design target. In other words, our goal is not only to propose a new propagation dynamics, but to specify and shape in advance which graph-frequency bands should decay, persist, or be selectively amplified as depth grows. This viewpoint yields a qualitatively different level of control over deep graph learning. It supports both controlled band-limited instability for graph pattern formation and near-critical yet stable propagation for preserving informative nonsynchronous components over long depths. Thus, our framework provides a unified stability-centered perspective that goes beyond architecture-specific PDE design and directly connects spectral modes, dynamical behavior, and deep representation quality.

\section{Preliminaries on Deep Graph Dynamics}
\label{sec:prelim}

Through the lens of graph-coupled dynamical systems,
many modern GNNs can be viewed as iterating a shared node-wise update together with a propagation operator, where depth plays the role of time.
This perspective is useful for analyzing stability and  describing how node representations evolve under repeated propagation.
Our theory targets \emph{general} nonlinear graph-coupled dynamics.
As a running example, we employ reaction--diffusion dynamics~\cite{chamberlain2021grand} throughout the paper.

\subsection{Graphs and graph-coupled propagation}
\label{sec:graphs-and-graph-coupled-propagation}

Let $\mathcal G=(\mathcal V,\mathcal E)$ be an undirected weighted graph with $|\mathcal V|=n$,
node set $\mathcal V=\{1,\dots,n\}$, and adjacency $\mathbf A\in\mathbb R^{n\times n}$.

Let $\mathbf d=\mathbf A\mathbf 1$ be the degree vector, where $\mathbf 1\in\mathbb R^n$ is the all-ones vector, and let $\mathbf D=\mathrm{diag}(\mathbf d)$.
We use a normalized propagation backbone $\widetilde{\mathbf S}\in\mathbb R^{n\times n}$, for example
$\widetilde{\mathbf S}_{\mathrm{sym}}=\mathbf D^{-1/2}\mathbf A\mathbf D^{-1/2}$ or
$\widetilde{\mathbf S}_{\mathrm{rw}}=\mathbf D^{-1}\mathbf A$.
We define the Laplacian-type coupling
\begin{equation}
\mathbf G := \widetilde{\mathbf S}-\mathbf I_n,
\end{equation}
so that $-\mathbf G=\mathbf I_n-\widetilde{\mathbf S}$ acts as a normalized Laplacian, where $\mathbf I_n$ is the $n\times n$ identity.

When $\widetilde{\mathbf S}=\widetilde{\mathbf S}_{\mathrm{rw}}$, we have the zero-row-sum property $\mathbf G\mathbf 1=\mathbf 0$.
For symmetric normalization, the corresponding dynamics are spectrally equivalent via a similarity transform
(see Appendix~\ref{app:lem:norm-vs-rw-main}), so the mode-wise stability conclusions in this paper remain unchanged.

\subsubsection{Node states.}
Let $\mathbf h_i(t)\in\mathbb R^q$ be the state at node $i$ at time $t$.
We stack all node states as
$
\mathbf H(t)\coloneqq[\mathbf h_1(t);\dots;\mathbf h_n(t)]\in\mathbb R^{nq}.
$
When convenient, we reshape $\mathbf H(t)$ into a matrix $\mathbf H_{\mathrm{mat}}(t)\in\mathbb R^{n\times q}$
whose $i$-th row is $\mathbf h_i(t)^\top$.

\subsubsection{Spectrum and graph frequencies.}
We work with the eigenvalues of $-\mathbf G=\mathbf I_n-\widetilde{\mathbf S}$.
For common normalized backbones on undirected graphs, these eigenvalues are real and nonnegative.
We denote them by $\{ \mu_k \}_{k=1}^n$ and order them with $\mu_1=0$ for a connected graph.
These scalars will index the mode-wise stability analysis in Section~\ref{sec:theory}.

\subsection{Heat diffusion view of over-smoothing}
\label{sec:prelim-heat}

We show the simplest continuous-time dynamics behind over-smoothing by keeping only propagation.
A residual propagation layer (e.g., GCN~\cite{GCN2017}) can be written in matrix form as
\begin{equation}
\mathbf H_{\mathrm{mat}}^{\ell+1}
=
\mathbf H_{\mathrm{mat}}^\ell
+
\tau\,\mathbf G\,\mathbf H_{\mathrm{mat}}^\ell,
\label{eq:disc-heat-G}
\end{equation}
where $\tau>0$ is a step size that maps layer depth to diffusion time.
Equation~\eqref{eq:disc-heat-G} is the explicit Euler discretization of Equation~\eqref{eq:cont-heat-G}
\begin{equation}
\frac{\mathrm d \mathbf H_{\mathrm{mat}}(t)}{\mathrm dt}
=
\mathbf G\,\mathbf H_{\mathrm{mat}}(t).
\label{eq:cont-heat-G}
\end{equation}
The solution is
\begin{equation}
\mathbf H_{\mathrm{mat}}(t)
=
e^{t\mathbf G}\,\mathbf H_{\mathrm{mat}}(0),
\label{eq:heat-kernel-G}
\end{equation}
so propagation applies a diffusion kernel $e^{t\mathbf G}$.

Since the eigenvalues of $-\mathbf G$ are $\{\mu_k\}$, the eigenvalues of $\mathbf G$ are $\{-\mu_k\}$, so each graph-frequency component is scaled by $e^{-t\mu_k}$. For common normalized backbones on undirected graphs, $\mu_k\in[0,2]$~\cite{Zhu2025SDMG}.
Thus larger $\mu_k$ (equivalently more negative eigenvalues of $\mathbf G$) decay faster, quickly removing high-frequency variation and leaving increasingly smooth, low-frequency representations as depth/time increases.

\subsection{Reaction--diffusion dynamics as a running example}
\label{sec:general-rd}

We consider graph reaction--diffusion dynamics as the ground of this paper:
\begin{equation}
\frac{\mathrm d \mathbf h_i(t)}{\mathrm dt}
=
f\big(\mathbf h_i(t)\big)
+
\sum_{j=1}^n G_{ij}\,\mathbf C\,\mathbf h_j(t),
\label{eq:rd-first}
\end{equation}
where $f:\mathbb R^q\to\mathbb R^q$ is a shared node-wise nonlinear update and
$\mathbf C\in\mathbb R^{q\times q}$ couples feature channels in the propagation term.
Setting $f\equiv 0$ recovers purely linear message passing, including diffusion and heat-kernel style propagation, as we showed in Section~\ref{sec:prelim-heat}.

Equation~\eqref{eq:rd-first} contains two ingredients common to dynamical-system-based GNNs.
The local term $f$ acts independently at each node, while the coupling term mixes information through the propagation backbone.
Although we use reaction--diffusion dynamics as a representative template, our following stability analysis applies to a broad class of nonlinear graph-coupled dynamics that preserve synchronized behavior.

\begin{figure*}[t] 
\vskip -0.0in
\begin{center}
\includegraphics[width=1.0\linewidth]{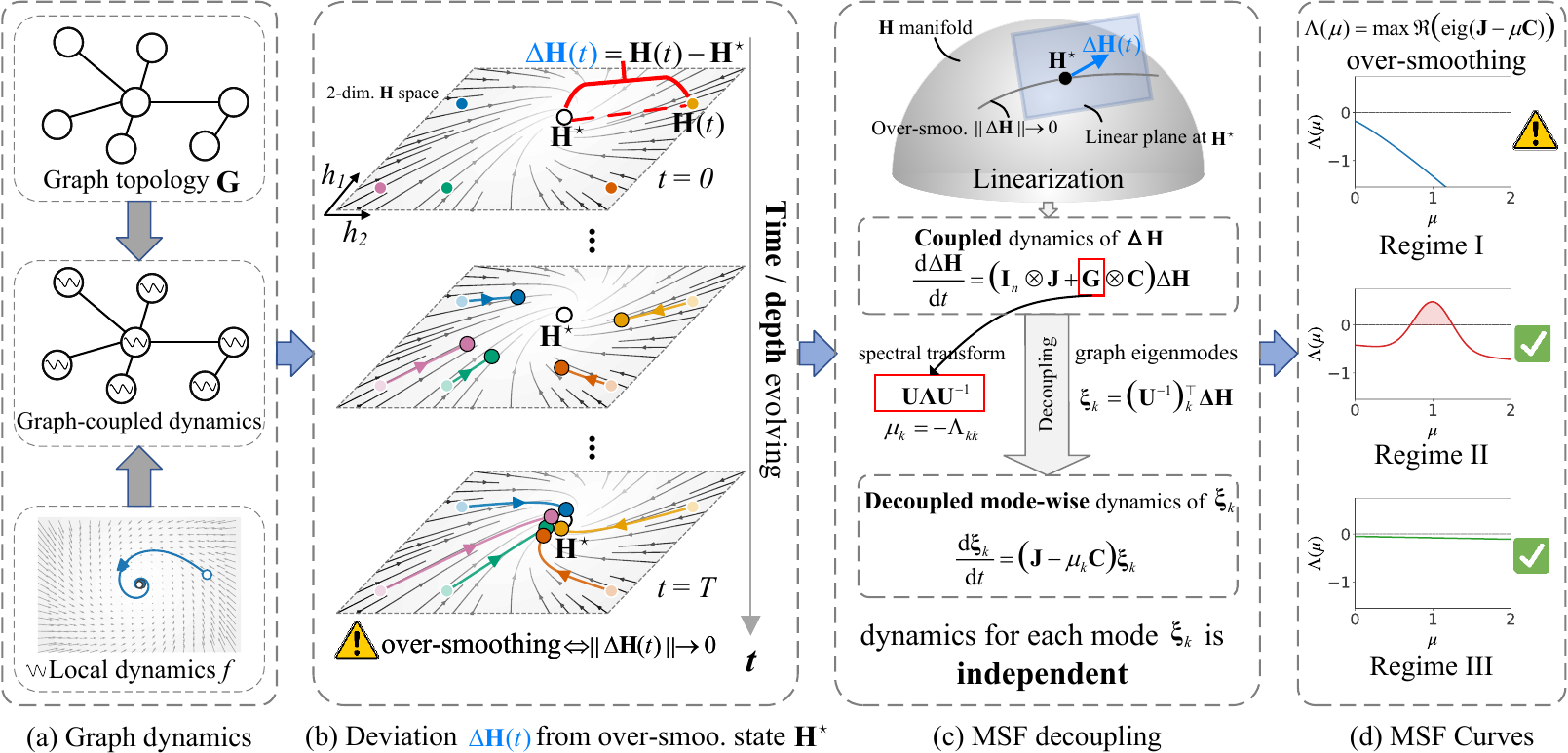}
\caption{
Overview of our theory.
\textbf{(a)} A deep GNN can be viewed as \emph{local update} $f$ coupled by \emph{graph mixing} $\mathbf G$.
\textbf{(b)} We study over-smoothing by tracking the deviation from an all-nodes-equal state $\mathbf H^\star$:
$\Delta\mathbf H(t)=\mathbf H(t)-\mathbf H^\star$; over-smoothing means $\|\Delta\mathbf H(t)\|\to 0$ as depth grows.
\textbf{(c)} Linearizing at $\mathbf H^\star$ and diagonalizing $\mathbf G$ turn the coupled system into independent graph-frequency modes
$\dot{\boldsymbol\xi}_k=(\mathbf J-\mu_k\mathbf C)\boldsymbol\xi_k$.
\textbf{(d)} The master stability curve $\Lambda(\mu)$ predicts per-mode growth/decay and highlights three behaviors:
(I) strong decay (diffusion over-smoothing), (II) band-limited growth (patterning), and (III) weak decay (near-critical, deep but stable).
}

\label{fig:theory}
\end{center}
\vskip -0.0in
\end{figure*}
\section{Theory for Deep Propagation}
\label{sec:theory}

We develop a spectral mode-wise stability theory for deep graph propagation.
Our goal is to predict when repeated propagation suppresses, preserves, or amplifies non-constant structure, and to translate this understanding into design objectives.
Fig.~\ref{fig:theory} summarizes the main logic of our framework, from graph-coupled nonlinear dynamics and deviations from synchrony to mode-wise decoupling and the resulting master stability regimes.

\subsection{Synchronized manifold.}
Assume $\mathbf G\mathbf 1=\mathbf 0$.
Then the synchronized set
$
\mathcal S=\{\mathbf 1\otimes \mathbf h:\ \mathbf h\in\mathbb R^q\}
$
is invariant under the coupled dynamics.
Here $\otimes$ denotes the Kronecker product, so $\mathbf 1\otimes \mathbf h$ repeats the same vector at every node.
Thus, if all nodes start from the same state, they remain identical over time.
We next study how small deviations from $\mathcal S$ evolve, connecting deep propagation to stability and over-smoothing.

\subsection{Over-smoothing: non-synchronous decay}
\label{sec:msf-oversmoothing}
We formalize over-smoothing as the contraction of small deviations away from the synchronized manifold $\mathcal S$. 
To measure such deviations, we pick a reference synchronous solution and study the evolution of small perturbations around it.

Suppose that the node dynamics $\frac{\mathrm d \mathbf h_i(t)}{\mathrm dt}=f(\mathbf h_i(t))$ have an equilibrium $\mathbf h^\star$ with $f(\mathbf h^\star)=\mathbf 0$.
Then $\mathbf H^\star=\mathbf 1\otimes \mathbf h^\star\in\mathbb R^{nq}$ is a synchronous equilibrium of Equation~\eqref{eq:rd-first}.
We study deviations $\Delta\mathbf H(t)=\mathbf H(t)-\mathbf H^\star$, linearize Equation~\eqref{eq:rd-first} around $\mathbf H^\star$, and obtain
\begin{equation}
\frac{\mathrm d\,\Delta \mathbf H}{\mathrm dt}
=
\bigl(\mathbf I_n\otimes \mathbf J + \mathbf G\otimes \mathbf C\bigr)\Delta\mathbf H,
\label{eq:variational-main-rewrite}
\end{equation}
where $\mathbf J=\nabla f(\mathbf h^\star)$ is the Jacobian of the local reaction
(see Appendix~\ref{app:msf-variational} for the full derivation).

Over-smoothing corresponds to the decay of deviations orthogonal to the synchronized manifold.
Let $\mathbf \Pi=\mathbf I_n-\frac{1}{n}\mathbf 1\mathbf 1^\top$ denote the projection that removes the synchronous component.

\begin{definition}[Over-smoothing as decay of non-synchronous deviations]
\label{def:oversmoothing}
Let $\Delta\mathbf H(t)$ follow Equation~\eqref{eq:variational-main-rewrite} and define
$\Delta\mathbf H_\perp(t)=(\mathbf \Pi\otimes \mathbf I_q)\Delta\mathbf H(t)$.
We say the dynamics exhibit an over-smoothing if $\|\Delta\mathbf H_\perp(t)\|\to 0$ as $t\to\infty$,
i.e., all node-wise differences orthogonal to the synchronous manifold vanish with depth.
\end{definition}

Thus, understanding over-smoothing reduces to understanding when the non-synchronous component
$\Delta\mathbf H_\perp(t)$ contracts under Equation~\eqref{eq:variational-main-rewrite}, and how fast it contracts.

\subsection{Master stability decomposition}
\label{sec:general-msf}
Equation~\eqref{eq:variational-main-rewrite} couples graph topology and feature dynamics in an $nq$-dimensional linear system.
We decouple this system by expanding $\Delta\mathbf H$ in the eigenbasis of $\mathbf G$, as illustrated in Fig.~\ref{fig:theory}(c).
This yields independent dynamics for each graph mode, and produces a scalar stability curve that depends only on the local Jacobian $\mathbf J$, the coupling $\mathbf C$, and a graph eigenvalue $\mu$. The corresponding three regimes are sketched in Fig.~\ref{fig:theory}(d).

\begin{theorem}[General master-stability decomposition at equilibrium]
\label{thm:general-msf}
Assume $\mathbf G$ is diagonalizable with eigen-decomposition
$\mathbf G=\mathbf U\boldsymbol\Lambda \mathbf U^{-1}$,
$\boldsymbol\Lambda=\mathrm{diag}(\kappa_1,\dots,\kappa_n)$,
and define $\mu_k:=-\kappa_k\ge 0$.
Let the modal coordinates be $\boldsymbol\xi=(\mathbf U^{-1}\otimes \mathbf I_q)\Delta\mathbf H$.
Then each mode $\boldsymbol\xi_k\in\mathbb R^q$ satisfies
\begin{equation}
\frac{\mathrm d \boldsymbol\xi_k}{\mathrm dt}
=
\bigl(\mathbf J-\mu_k \mathbf C\bigr)\boldsymbol\xi_k,
\qquad k=1,\dots,n.
\label{eq:mode-ode-mu}
\end{equation}
\end{theorem}

The proof is given in Appendix~\ref{app:mode-odes}.
Equivalently, the variational dynamics decompose into $n$ independent $q$-dimensional linear systems
\begin{equation}
\frac{\mathrm d \boldsymbol\xi}{\mathrm dt}
=
\mathrm{diag}\bigl(\mathbf J-\mu_1\mathbf C,\dots,\mathbf J-\mu_n\mathbf C\bigr)\,\boldsymbol\xi .
\label{eq:modal-block-diag}
\end{equation}

Theorem~\ref{thm:general-msf} reduces stability to the spectra of the mode matrices $\mathbf J-\mu_k\mathbf C$.
We therefore define the master stability function (MSF) for deep propagation
\begin{equation}
\Lambda(\mu)\coloneqq \max \Re\bigl(\mathrm{eig}(\mathbf J-\mu\mathbf C)\bigr).
\label{eq:msf-def}
\end{equation}
For a connected graph, $\mu_1=0$ corresponds to the synchronous mode and modes with $k\ge2$ are non-synchronous.
Moreover, $\Delta\mathbf H_\perp$ in Definition~\ref{def:oversmoothing} has support only on modes $k\ge2$.
Hence, if $\Lambda(\mu_k)<0$ for all $k\geq 2$, then all non-synchronous modes are contractive, so node-wise variations decay with depth and eventually collapse.

\subsection{MSF regimes and stability-oriented objectives}
\label{sec:three-regimes}
The sign and magnitude of $\Lambda(\mu)$ in Equation~\eqref{eq:msf-def} define three regimes that we target in our deep GNN layer designs.

\subsubsection{(I) Dissipative regime (diffusion-only).}
When $f\equiv 0$, we have $\mathbf J=\mathbf 0$ and therefore $\Lambda(\mu) = \mu \Lambda(1) \le 0$ for all $\mu\ge 0$.
Larger $\mu$ yields more negative margins, so higher-frequency non-synchronous modes decay faster.
This recovers the over-smoothing behavior of a heat-kernel-style message passing, such as GCN. As, for example, the analysis in \cite{cai2020note} and \cite{wu2023demystifying} shows, many typical activation functions are 1-Lipschitz and thus do not materially change this smoothing dynamic.

\subsubsection{(II) Diffusion-induced (DI) instability regime.}
A diffusion-induced instability occurs when the local reaction is stable at $\mu=0$ but becomes unstable over an intermediate range of $\mu$
\begin{equation}
\Lambda(0)<0,
\qquad
\exists\,\mu\in[\mu_a,\mu_b]\ \text{such that}\ \Lambda(\mu)>0,
\label{eq:turing-band}
\end{equation}
and $\Lambda(\mu)<0$ outside this interval.
On a graph, $\mu$ is instantiated by the discrete spectrum $\{\mu_k\}$ of $-\mathbf G$.
Pattern-forming growth occurs when some non-synchronous mode enters the positive-MSF band, meaning that there exists $k\ge2$ with $\Lambda(\mu_k)>0$.
To make this criterion verifiable and trainable, we specialize $\mathbf J-\mu\mathbf C$ to a $2\times2$ activator--inhibitor block with anisotropic diagonal diffusion, where instability reduces to a determinant test and yields an explicit unstable interval $(\mu_-,\mu_+)$ (Corollary~\ref{cor:turing-2x2-band} in Appendix~\ref{app:di-msf-band}).

\subsubsection{(III) Near-critical (NC) stable regime.}
To enable deep propagation without collapse, we want all non-synchronous modes to be stable while remaining only weakly contractive
\begin{equation}
-\delta_{\text{MSF}} \le \Lambda(\mu_k) \le -\epsilon_{\text{MSF}} <0,
\qquad k\ge 2,
\label{eq:narrow-band}
\end{equation}
with small $\delta_{\text{MSF}}$ and potentially very small $\epsilon_{\text{MSF}}$.
In discrete time, this corresponds to layer maps whose spectral radii are close to one on modes with $k\ge2$.
This mitigates over-smoothing while preserving stability.

\section{Towards Practical Architectural Design}
\label{sec:design}

Section~\ref{sec:theory} turns layer design into stability shaping by controlling the MSF curve $\Lambda(\mu)$.
Consequently, many architectures can avoid over-smoothing as long as they operate in Regime~II 
or Regime~III 
.
We instantiate this principle with two variants: SDGL-DI based on Schnakenberg-type~\cite{schnakenberg1979simple} activator--inhibitor dynamics 
(Regime~II) 
and SDGL-NC based on predator--prey dynamics~\cite{takeuchi1996global}
(Regime~III).

\subsection{General layer formulation}
\label{sec:design-general-layer}

Both SDGL variants share a unified coupled-field architecture obtained by discretizing
a two-field reaction--diffusion template.
Let $\mathbf X\in\mathbb R^{n\times d_{\text{in}}}$ be input node features on a graph with $n$ nodes.
A learnable lift (e.g., two heads of an MLP encoder) maps $\mathbf X$ into two interacting fields
$\mathbf R^0,\mathbf F^0\in\mathbb R^{n\times d}$, where $d$ is the hidden width per field.
We fix a diffusion backbone $\widetilde{\mathbf S}\in\mathbb R^{n\times n}$ and use
$\mathbf G=\widetilde{\mathbf S}-\mathbf I_n$.

At layer $\ell$, SDGL applies a coordinate-wise reaction and an anisotropic diffusion update
\begin{align}
\mathbf R^{\ell+1} &= \mathbf R^\ell
+ \Delta t\Big(f_R(\mathbf R^\ell,\mathbf F^\ell;\boldsymbol\theta)
+ C_R\,\mathbf G\mathbf R^\ell\Big),
\label{eq:gen_update_R}\\
\mathbf F^{\ell+1} &= \mathbf F^\ell
+ \Delta t\Big(f_F(\mathbf R^\ell,\mathbf F^\ell;\boldsymbol\theta)
+ C_F\,\mathbf G\mathbf F^\ell\Big),
\label{eq:gen_update_F}
\end{align}
where $\Delta t>0$ is the step size, $\boldsymbol\theta$ collects the learnable parameters in the reaction terms,
and $C_R,C_F\in\mathbb R_{>0}$ are learnable diffusion strengths.
The reaction terms $f_R,f_F$ act row-wise and channel-wise, meaning that the update of channel $j$ at node $i$
depends only on $(\mathbf r_{ij}^\ell,\mathbf f_{ij}^\ell)$.
Linearizing the layer around a synchronous equilibrium yields the per-mode matrices
$\mathbf J-\mu\mathbf C$ from Section~\ref{sec:theory}, so the architectural design reduces to stability shaping.
The unified layer arises from an explicit Euler discretization of a two-field graph reaction--diffusion system. A compact derivation is provided below, and additional algebraic details are deferred to Appendix~\ref{app:rd-discretization}.

\subsubsection{Continuous-time origin of the layer.}
The unified SDGL layer can be viewed as an explicit Euler discretization of a two-field graph reaction--diffusion system
\begin{align}
\frac{\mathrm d\mathbf R}{\mathrm dt}
&=
f_{\mathbf R}(\mathbf R,\mathbf F;\boldsymbol\theta)
+ C_R\,\mathbf G\mathbf R,\\
\frac{\mathrm d\mathbf F}{\mathrm dt}
&=
f_{\mathbf F}(\mathbf R,\mathbf F;\boldsymbol\theta)
+ C_F\,\mathbf G\mathbf F,
\end{align}
where $f_{\mathbf R},f_{\mathbf F}$ are row-wise reaction maps shared across nodes.
Applying explicit Euler with step size $\Delta t$ at $t_\ell=\ell\Delta t$ yields
Equations~\eqref{eq:gen_update_R}--\eqref{eq:gen_update_F}.
This makes the architectural design directly compatible with the continuous-time MSF analysis in Section~\ref{sec:theory}.

\subsubsection{Output representation.}
After $L$ layers, we concatenate the two final fields along the feature dimension and form
\begin{equation}
\mathbf H_{\mathrm{mat}} \coloneqq [\mathbf R^{L}||\mathbf F^{L}] \in \mathbb R^{n\times 2d},
\end{equation}
where $[\cdot||\cdot]$ means concatenation.
This matches the two-field node state $\mathbf h_i=[\mathbf r_i||\mathbf f_i]$ used in the MSF analysis
and retains both channels for downstream prediction.

\subsection{SDGL-DI layer (DI patterning, Regime~II)}
\label{sec:design-di-layer}

SDGL-DI instantiates $f_R,f_F$ by a Schnakenberg-type activator--inhibitor reaction.
The following parameters are learnable and shared across nodes: 
$\boldsymbol\psi,\boldsymbol\chi\in\mathbb R^d_{>0}$ and $\eta\in\mathbb R_{>0}$.
We shift the variables such that the synchronous equilibrium is exactly
$\mathbf R=\mathbf 0$ and $\mathbf F=\mathbf 0$.
Details are given in Appendix~\ref{app:di-schnakenberg}.
$\odot$ denotes elementwise multiplication, and $\oslash$ denotes elementwise division.

Define
$\mathbf p^\star=\boldsymbol\psi+\boldsymbol\chi\in\mathbb R^d_{>0}$ and
$\mathbf v^\star=\boldsymbol\chi\oslash\big(\eta \mathbf p^\star \odot \mathbf p^\star\big)\in\mathbb R^d_{>0}$.
We add these offsets to every node and set
$\mathbf P^\ell=\mathbf R^\ell+\mathbf p^\star$,
$\mathbf V^\ell=\mathbf F^\ell+\mathbf v^\star$, and
$\boldsymbol\zeta^\ell=\eta\mathbf P^\ell \odot \mathbf P^\ell\odot \mathbf V^\ell$.
The SDGL-DI layer is then
\begin{align}
\mathbf R^{\ell+1}
&=
\mathbf R^\ell
+
\Delta t\Big(
\boldsymbol\psi-\mathbf P^\ell+\boldsymbol\zeta^\ell
+
C_R\,\mathbf G\mathbf R^\ell
\Big),
\label{eq:sdgl-di-layer-r}\\
\mathbf F^{\ell+1}
&=
\mathbf F^\ell
+
\Delta t\Big(
\boldsymbol\chi-\boldsymbol\zeta^\ell
+
C_F\,\mathbf G\mathbf F^\ell
\Big).
\label{eq:sdgl-di-layer-f}
\end{align}
Under the MSF, this choice produces a diffusion-induced, band-limited unstable window (Regime~II).
The continuous-time ODE formulation and the resulting closed-form band criterion are provided in Appendix~\ref{app:di-schnakenberg} and~\ref{app:di-msf-band}.

\subsection{SDGL-NC layer (Near-critical, Regime~III)}
\label{sec:design-nc-layer}

SDGL-NC instantiates $f_R$ and $f_F$ using a predator--prey reaction,
such that the resulting mode matrices $\mathbf J-\mu\mathbf C$ keep all non-synchronous modes stable while remaining near-critical.
This yields a deep propagation with mild contraction in Regime~III.

\subsubsection{Why NC is near-critical.}
Ignoring diffusion, the coordinate-wise predator--prey reaction has a nontrivial coexistence equilibrium.
For each feature coordinate $j$, it satisfies
\begin{equation}
r_j^\star=\frac{\gamma_j}{\delta_j},
\qquad
f_j^\star=\frac{\alpha_j}{\beta_j}.
\end{equation}
At this equilibrium, the local $2\times2$ Jacobian block is
\begin{equation}
\mathbf J_j=
\begin{bmatrix}
0 & -\beta_j\gamma_j/\delta_j\\
\delta_j\alpha_j/\beta_j & 0
\end{bmatrix},
\end{equation}
whose eigenvalues are
\begin{equation}
\lambda_{j,\pm}=\pm i\sqrt{\alpha_j\gamma_j}.
\end{equation}
Therefore, the uncoupled node dynamics are locally critical, and diffusion shifts the modal dynamics to $\mathbf J-\mu\mathbf C$, yielding stable but only weakly contractive propagation. See Appendix~\ref{app:nc-imag-spectrum} for more details.

The following parameters are learnable and shared across nodes:
$\boldsymbol\alpha,\boldsymbol\beta,\boldsymbol\delta,\boldsymbol\gamma\in\mathbb R^d_{>0}$.
The SDGL-NC layer is then
\begin{align}
\mathbf R^{\ell+1}
&=
\mathbf R^\ell
+ \Delta t\Big(
\boldsymbol\alpha\odot\mathbf R^\ell
-\boldsymbol\beta\odot\mathbf R^\ell\odot\mathbf F^\ell
+ C_R\,\mathbf G\mathbf R^\ell
\Big),
\notag
\\
\mathbf F^{\ell+1}
&=
\mathbf F^\ell
+ \Delta t\Big(
\boldsymbol\delta\odot\mathbf R^\ell\odot\mathbf F^\ell
-\boldsymbol\gamma\odot\mathbf F^\ell
+ C_F\,\mathbf G\mathbf F^\ell
\Big).
\notag
\end{align}

In principle, a sufficient step-size condition ensures that the per-mode linearized transfer map has the spectral radius in
$[1-\varepsilon_{\text{MSF}},1)$ on all non-synchronous modes.
For a graph frequency $\mu$, the per-mode map is
$
T_{\mathrm{NC}}(\mu)=\mathbf I_{2d}+\Delta t\bigl(\mathbf J-\mu\mathbf C\bigr).
$
We formalize this in Theorem~\ref{thm-near-unit-radius}.

\begin{theorem}[Spectral radius close to one for NC layers]
\label{thm-near-unit-radius}
Assume positive predator--prey parameters, such that the node-wise reaction admits a coexistence equilibrium and the Jacobian $\mathbf J$ at this equilibrium has purely imaginary eigenvalues in each $2\times 2$ feature block.
Let $\mathbf C=\mathrm{diag}(C_R\mathbf I_d,\,C_F\mathbf I_d)$ with $C_R,C_F>0$.
Let $\{\mu_k\}_{k=1}^n$ be the nonnegative eigenvalues of $-\mathbf G$, ordered so that $\mu_1=0$ for a connected graph.
Then for every non-synchronous mode $k\ge 2$ and any tolerance $\varepsilon_{\text{MSF}}\in(0,1)$,
there exists $\Delta t_\varepsilon>0$ such that for all $0<\Delta t\le \Delta t_\varepsilon$,
\begin{equation}
1-\varepsilon_{\text{MSF}}
\;\le\;
\rho\bigl(T_{\mathrm{NC}}(\mu_k)\bigr)
\;<\;
1,
\end{equation}
where $\rho(\cdot)$ denotes the spectral radius.
\end{theorem}

The proofs are given in Appendix~\ref{app:proof-near-unit-radius}.

\subsection{Qualitative view  of stability-shaped propagation}
\label{sec:design-qual-grid}

To build intuition for the behaviors targeted by Regimes~I--III, we provide a grid visualization of deep propagation.
We construct a $300\times 300$ grid graph, initialize a scalar field with i.i.d.\ Gaussian noise, and apply $150$ propagation layers.
Fig.~\ref{fig:turing_oversmooth} compares diffusion-style GCN propagation with our SDGL-DI and SDGL-NC under a shared color scale.

We found that GCN rapidly collapses to a nearly constant field, reflecting over-smoothing under repeated diffusion.
Our SDGL-DI produces coherent spatial structures with sharp interfaces, consistent with a diffusion-induced instability that amplifies a band of modes into Turing-like patterns.
SDGL-NC exhibits a distinct stable regime: it retains a visible texture and non-trivial variation after many layers, while avoiding large-scale pattern formation.
Overall, this qualitative visualization aligns with our stability-shaping objectives and complements the mode-wise MSF analysis.

\subsection{Complexity analysis}
\subsubsection{Time complexity.} We compare SDGL with a standard GCN-style message passing baseline. Let $n=|\mathcal V|$ be the number of nodes, let $m=\mathrm{nnz}(\widetilde{\mathbf S})$ be the number of nonzeros in the sparse diffusion backbone $\widetilde{\mathbf S}$ (hence $\mathbf G=\widetilde{\mathbf S}-\mathbf I$), let $d$ be the hidden width per field, and let $L$ be the number of layers. A GCN layer is dominated by one sparse--dense multiplication $\widetilde{\mathbf S}\mathbf H^\ell$ with $\mathbf H^\ell\in\mathbb R^{n\times d}$, giving $O(md)$ time per layer and $O(Lmd)$ overall. Each SDGL layer evolves two fields $\mathbf R^\ell,\mathbf F^\ell\in\mathbb R^{n\times d}$ through coordinate-wise reactions ($O(nd)$) and two diffusion multiplications ($2\cdot O(md)$), yielding $O(2md+nd)=O(md)$ per layer on sparse graphs and $O(Lmd)$ over $L$ layers, i.e., the same linear scaling in $m$ as GCN with only a small constant-factor overhead from maintaining two coupled fields. Backpropagation follows the same asymptotic order.
\begin{figure}[t] 
\begin{center}
\includegraphics[width=1.0\linewidth]{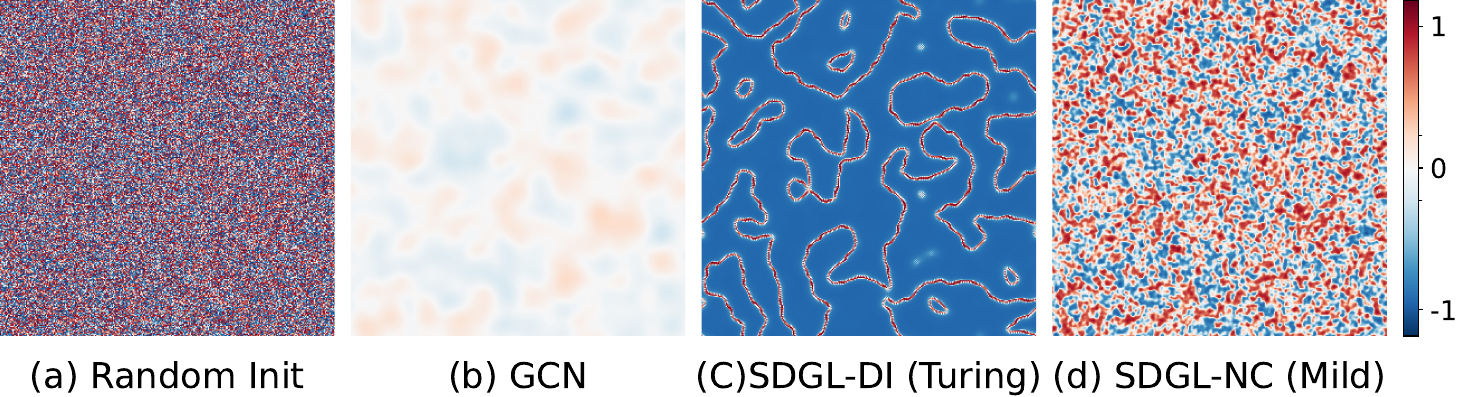}
\caption{Qualitative dynamics analysis.
Starting from random noise, we run $150$ layers and visualize the resulting scalar field with a shared color scale.
GCN collapses toward a near-constant field (over-smoothing).
Our SDGL-DI produces sharp Turing-like patterns via diffusion-induced instability.
SDGL-NC preserves non-trivial variation under stable near-critical propagation.}

\label{fig:turing_oversmooth}
\end{center}
\vskip -0.0in
\end{figure}

\subsubsection{Space complexity.}
GCN stores $L$ layer-wise node embeddings $\mathbf H^\ell\in\mathbb R^{n\times d}$ for backpropagation, hence activation memory $O(Lnd)$ (and $O(nd)$ for inference). SDGL maintains two coupled fields $\mathbf R^\ell,\mathbf F^\ell\in\mathbb R^{n\times d}$, so the activation memory becomes $O(2Lnd)=O(Lnd)$, i.e., a constant-factor ($\approx2\times$) overhead. In both cases, storing the sparse backbone $\widetilde{\mathbf S}$ requires $O(m)$ memory. Parameter memory is dominated by the input lift (e.g., $O(d_{\text{in}}d)$ for a linear map, up to constant factors for two heads), while reaction/diffusion parameters are coordinate-wise and contribute only $O(d)$ plus $O(1)$ scalars.

\section{Experiments}
We empirically validate SDGL from four complementary perspectives.
First, we evaluate whether stability-shaped propagation improves predictive performance on diverse real-world node-classification (Section~\ref{sec:results_node_classification}) and graph-level classification (Section~\ref{sec:results-graph-classification}) benchmarks.
Second, we test whether SDGL exhibits better depth scaling at the task level than diffusion-dominated baselines as the number of layers increases (Section~\ref{sec:depth-scaling}).
Third, we examine whether the observed empirical behavior is consistent with our theoretical picture through mechanism-oriented diagnostics, including Dirichlet energy evolution (Section~\ref{sec:dirichlet-energy}) and spectral validation of near-critical propagation (Section~\ref{sec:val-near-critical}).
Fourth, we assess whether SDGL remains parameter-efficient as depth increases (Section~\ref{sec:param_count}).

\subsection{Experimental setup}

\subsubsection{Datasets}
\textbf{Node classification.}
We evaluate SDGL on standard transductive node classification benchmarks covering both homophilic and heterophilic settings.
The homophilic datasets include the citation and co-purchase/co-author benchmarks Cora, CiteSeer, PubMed, Coauthor-CS, Computers, and Photo~\cite{sen2008collective,Zhu2025SDMG}.
The heterophilic datasets include Actor, Squirrel, and Chameleon~\cite{ding2025let}.
These benchmarks are particularly relevant to our study because, in heterophilic graphs, local neighbors are often not label-consistent, so shallow smoothing-based propagation can easily wash out useful distinctions.
Instead, informative signals may only become accessible through multi-hop propagation, which makes long-range dependency modeling a central challenge.
For these datasets, we follow the commonly used public splits and evaluation protocols adopted in prior work~\cite{ding2025let}.

\textbf{Graph classification with long-range dependencies.}
We further evaluate graph-level prediction on FakeNewsNet~\cite{dou2021user}, using its two standard subsets, \emph{PolitiFact} and \emph{GossipCop}.
In this setting, each sample consists of a news article, its associated social-context graph, and a binary real/fake label.
These benchmarks are particularly relevant to our study because useful evidence may be distributed across multiple hops in the propagation graph~\cite{zhu2024propagation}.
We follow~\cite{dou2021user} for preprocessing.
Table~\ref{tab:dataset-stats} summarizes the basic dataset statistics, while additional preprocessing details and hyperparameter disclosure are provided in Appendix~\ref{app:dataset-stats} and Appendix~\ref{sec:hyperparameter_disclosure}, respectively.
\begin{table*}[t]
\caption{Classification accuracy (\%) on homophilic datasets. We report mean $\pm$ std.
\textbf{Best} in bold, \underline{second-best} underlined. ``--'' denotes data that is out of memory or not included in the original paper}
\label{tab:homophilic-acc}
\vskip 0.15in
\begin{center}
\begin{small}
\begin{tabular}{lcccccc}
\toprule
\textbf{Dataset} & \textbf{Cora} & \textbf{CiteSeer} & \textbf{PubMed} & \textbf{Coauthor-CS} & \textbf{Computers} & \textbf{Photo} \\
\midrule
GAT &
$83.03\pm0.72$ & $72.53\pm0.54$ & $79.0\pm0.30$ & $89.51\pm0.54$ & $86.90\pm0.32$ & $92.56\pm0.40$ \\
GCN &
$81.53\pm0.53$ & $70.32\pm0.58$ & \underline{$79.04\pm0.37$} & $90.98\pm0.42$ & $86.49\pm0.51$ & $92.42\pm0.23$ \\
GraphSAGE &
$79.37\pm1.70$ & $67.31\pm1.63$ & $75.52\pm2.19$ & $90.62\pm0.42$ & $76.42\pm7.60$ & $88.71\pm2.68$ \\
SAN &
$77.98\pm1.31$ & $66.30\pm0.88$ & $74.68\pm0.81$ & $94.51\pm0.15$ & $89.93\pm0.16$ & $94.86\pm0.10$ \\
GTN &
$78.60\pm1.72$ & $68.80\pm1.69$ & $76.90\pm0.77$ & $-$ & $-$ & $-$ \\
\midrule
GRAND &
$82.83\pm1.31$ & $70.26\pm1.46$ & $78.89\pm1.96$ & $90.10\pm0.78$ & $82.79\pm0.54$ & $91.51\pm0.41$ \\
KuramotoGNN &
$81.16\pm1.61$ & $70.40\pm1.02$ & $78.69\pm1.91$ & $91.05\pm0.56$ & $80.06\pm1.60$ & $92.77\pm0.42$ \\
GraphCON &
$83.45\pm0.56$ & \bm{$73.94\pm1.63$} & $78.44\pm0.34$ & $90.30\pm0.74$ & $82.76\pm0.58$ & $91.78\pm0.50$ \\
BIMP &
$83.19\pm1.13$ & $71.09\pm1.40$ & \bm{$80.16\pm2.03$} & $92.48\pm0.26$ & $84.73\pm0.61$ & $92.90\pm0.44$ \\
\midrule
\textbf{SDGL-DI (Ours)} &
$\underline{83.21 \pm 1.48}$ & $70.03 \pm 1.32$ & $78.40 \pm 0.27$ & \underline{$94.62 \pm 0.67$} & \bm{$90.94 \pm 0.48$} & \underline{$95.03 \pm 0.77$} \\
\textbf{SDGL-NC (Ours)} &
$\bm{84.03 \pm 1.53}$ & \underline{$72.61 \pm 0.86$} & $78.52 \pm 0.86$ & \bm{$95.68 \pm 0.14$} & \underline{$90.39 \pm 0.25$} & \bm{$95.33 \pm 0.25$} \\
\bottomrule
\end{tabular}
\end{small}
\end{center}
\vskip -0.1in
\end{table*}

\begin{table}[t]
    \centering
    \caption{\textbf{Dataset statistics.}
    For node classification, we report the number of nodes (\#V), edges (\#E), input feature dimension, and classes.
    For FakeNewsNet graph classification, we report the corresponding statistics after preprocessing.}
    \label{tab:dataset-stats}
    \setlength{\tabcolsep}{5pt}
    \begin{tabular}{lrrrr}
        \toprule
        Dataset & \#V & \#E & \#Features & \#Classes \\
        \midrule
        \multicolumn{5}{l}{\textbf{Transductive node classification}}\\
        \midrule
        Cora              & 2{,}708  & 10{,}556  & 1{,}433 & 7 \\
        CiteSeer          & 3{,}327  & 9{,}104   & 3{,}703 & 6 \\
        PubMed            & 19{,}717 & 88{,}648  & 500     & 3 \\
        \midrule
        Amazon-Computers  & 13{,}752 & 491{,}722 & 767     & 10 \\
        Amazon-Photo      & 7{,}650  & 238{,}162 & 745     & 8 \\
        Coauthor-CS       & 18{,}333 & 163{,}788 & 6{,}805 & 15 \\
        \midrule
        Actor             & 7{,}600  & 27{,}544  & 931     & 5 \\
        Chameleon         & 2{,}277  & 38{,}328  & 2{,}325 & 5 \\
        Squirrel          & 5{,}201  & 217{,}073 & 2{,}089 & 5 \\
        \midrule
        \multicolumn{5}{l}{\textbf{Graph classification (FakeNewsNet)}}\\
        \midrule
        PolitiFact        & 41{,}054  & 40{,}740  & 768     & 2 \\
        GossipCop         & 314{,}262 & 308{,}798 & 768     & 2 \\
        \bottomrule
    \end{tabular}
\end{table}

\subsubsection{Baselines}
We compare SDGL against (i) \emph{heat-kernel-style message passing} GNNs, including GCN~\cite{GCN2017}, GAT~\cite{GAT2018}, and GraphSAGE~\cite{hamilton2017inductive}; (ii) \emph{Transformer-based} graph models that better capture long-range dependencies, including SAN~\cite{kreuzer2021rethinking} and GTN~\cite{yun2019graph}; and (iii) \emph{dynamics-inspired} deep propagation methods, including GRAND~\cite{chamberlain2021grand}, KuramotoGNN~\cite{nguyen2024coupled}, GraphCON~\cite{rusch2022graph}, and BIMP~\cite{wang2025resolving}. We also include GCNII~\cite{chen2020simple} as a residual-based deep GNN baseline. 

For FakeNewsNet graph classification, we follow prior protocols~\cite{dou2021user,zhu2024propagation} and compare with advanced baselines spanning two categories:
\textbf{(M1) content-based} models (CSI~\cite{CSI2017}, Bert-MLP~\cite{Bert2019}, SpaCy-MLP~\cite{Spacy2017}, UniPF~\cite{wei2022unified}, ChatGLM2-6B~\cite{zeng2023glmb});
and \textbf{(M2) propagation-structure-based} models, including long-sequence diffusion modeling methods (GCN~\cite{GCN2017}, GAT~\cite{GAT2018}, GraphGPS~\cite{GraphGPS2022}, BiGCN~\cite{BiGCN2020}, UPFD~\cite{UPFD2021}, EBGCN~\cite{EBGCN2021}, GSAT~\cite{GSAT2022}, GAMC~\cite{yin2024gamc}, PSGT~\cite{zhu2024propagation}).

\begin{table}[t]
\caption{Classification accuracy (\%) on heterophilic datasets. We report mean $\pm$ std.}
\label{tab:heterophilic-acc}
\vskip 0.15in
\centering
\setlength{\tabcolsep}{3pt}%
\resizebox{\columnwidth}{!}{%
\begin{tabular}{@{}lccc@{}}
\toprule
\textbf{Dataset} & \textbf{Actor} & \textbf{Squirrel} & \textbf{Chameleon} \\
\midrule
GAT &
$29.05\pm0.80$ & $30.08\pm1.03$ & $53.22\pm1.56$ \\
GCN &
$28.38\pm0.96$ & $28.87\pm1.55$ & $49.32\pm1.88$ \\
GraphSAGE &
$34.88\pm1.19$ & $36.90\pm1.03$ & $48.03\pm2.22$ \\
SAN &
$32.94\pm0.92$ & $37.00\pm1.28$ & $51.80\pm1.95$ \\
GTN &
$35.42\pm1.50$ & $36.45\pm1.72$ & $55.84\pm4.08$ \\
\midrule
GCNII &
$33.80\pm1.41$ & $38.74\pm1.20$ & $59.14\pm3.49$ \\
GRAND &
$33.34\pm1.16$ & $34.45\pm1.84$ & $38.43\pm3.33$ \\
KuramotoGNN &
$31.50\pm0.82$ & $35.02\pm1.67$ & $38.50\pm2.78$ \\
GraphCON &
$35.13\pm1.38$ & $26.90\pm2.17$ & $33.75\pm3.77$ \\
\midrule
\textbf{SDGL-DI (Ours)} &
\bm{$37.53\pm1.27$} & \underline{$41.21\pm1.17$}  & \underline{$60.31\pm1.53$}  \\
\textbf{SDGL-NC (Ours)} &
\underline{$37.02\pm1.07$} & \bm{$42.47\pm1.48$} & \bm{$61.62\pm1.88$} \\
\bottomrule
\end{tabular}%
}
\vskip -0.1in
\end{table}

\subsection{Results for node classification}
\label{sec:results_node_classification}

\subsubsection{Results on homophilic datasets.}
Table~\ref{tab:homophilic-acc} reports accuracy on six standard homophilic benchmarks.
Generally, both SDGL variants are competitive with, and often superior to, strong message-passing, Transformer-based, and dynamics-inspired baselines.
Among the two variants, SDGL-NC performs particularly well on homophilic graphs, achieving the highest accuracy on Cora, Coauthor-CS, and Photo, while SDGL-DI attains the best result on Computers and remains competitive on the remaining datasets.
These results are consistent with our stability perspective.
On homophilic graphs, informative signals are often concentrated in smoother components, so stable, near-critical propagation is especially beneficial because it preserves useful low-frequency variation without incurring the strong contraction associated with diffusion-dominated propagation.

\subsubsection{Results on heterophilic datasets.}
Table~\ref{tab:heterophilic-acc} summarizes the results on three widely used heterophilic benchmarks.
Our SDGL variants achieve the best performance on all three datasets, with SDGL-DI performing best on Actor and SDGL-NC performing best on Squirrel and Chameleon.
In particular, SDGL consistently improves over diffusion-dominated baselines as well as over deep residual propagation (GCNII) and prior dynamics-inspired models.
This behavior is again well aligned with our theory.
When labels are weakly aligned with graph structure, naive smoothing can easily mix incompatible neighborhood information and suppress informative non-synchronous variation.
In this regime, explicit control of mode-wise stability becomes particularly important, either by preserving variation under near-critical propagation or by avoiding excessive contraction through stability shaping.

In general, the node classification results show that SDGL is effective across both homophilic and heterophilic settings, while the relative strengths of SDGL-NC and SDGL-DI follow the expected propagation regimes.
The gains are most pronounced when deeper propagation is needed or when naive diffusion-style smoothing is especially harmful, which supports the central claim of this work: controlling mode-wise stability improves deep graph learning.
\begin{table}
\centering
  \caption{Performance comparison on \underline{PolitiFact} (M1: Content-based, M2: Propagation-structure-based methods).}
  \label{ta_p_overreusult}
  \begin{tabular}{l|lcccc}
    \toprule
     & \textbf{Methods} & \textbf{ACC} & \textbf{Pre} & \textbf{Rec} & \textbf{F1} \\
    \midrule
    \multirow{5}{*}{M1} & CSI & 0.734 & 0.672 & 0.550 & 0.688 \\
                        & Bert-MLP & 0.853 & 0.892 & 0.833 & 0.825 \\
                        & SpaCy-MLP & 0.375 & 0.124 & 0.136 & 0.272 \\
                        & UniPF & 0.760 & 0.783 & 0.800 & 0.754 \\
                        & ChatGLM2-6B & 0.856 & 0.847 & 0.856 & 0.850 \\   
    \cmidrule(lr){1-6}
    \multirow{8}{*}{M2} & GCN & 0.833 & \underline{0.958} & 0.766 & 0.831 \\
                        & GAT & 0.895 & 0.930 & 0.900 & 0.890 \\
                        & GraphGPS & 0.903 & 0.900 & 0.818 & 0.857 \\
                        & BiGCN & 0.854 & \textbf{0.960} & 0.800 & 0.851 \\
                        & UPFD & 0.875 & 0.900 & 0.900 & 0.866 \\
                        & EBGCN & 0.896 & 0.898 & 0.909 & 0.891 \\
                        & GSAT & 0.854 & 0.947 & 0.750 & 0.852 \\
                        & GAMC & 0.838 & 0.836 & 0.827 & 0.831 \\
                        
    \cmidrule(lr){1-6}
    \multirow{2}{*}{Ours} & \textbf{SDGL-DI} & \textbf{0.921} & 0.829 & \textbf{0.963} & \textbf{0.914} \\
     & \textbf{SDGL-NC} & \underline{0.904} & 0.815 & \underline{0.925} & \underline{0.896} \\
     
    \bottomrule
  \end{tabular}
\end{table}

\begin{table}
\centering
  \caption{Performance comparison on \underline{GossipCop} (M1: Content-based, M2: Propagation-structure-based methods).}
  \label{tab:gossipcop}
  \begin{tabular}{l|lcccc}
    \toprule
     & \textbf{Methods} & \textbf{ACC} & \textbf{Pre} & \textbf{Rec} & \textbf{F1} \\
    \midrule
    \multirow{5}{*}{M1} & CSI & 0.866 & 0.892 & 0.840 & 0.866 \\
                        & Bert-MLP & 0.962 & 0.968 & 0.954 & 0.961 \\
                        & SpaCy-MLP & 0.501 & 0.498 & 0.534 & 0.333 \\
                        & UniPF & 0.933 & 0.932 & 0.933 & 0.932 \\
                        & ChatGLM2-6B & 0.892 & 0.902 & 0.892 & 0.892 \\     
    \cmidrule(lr){1-6}
    \multirow{8}{*}{M2} & GCN & 0.957 & 0.931 & \textbf{0.984} & 0.957 \\
                        & GAT & 0.961 & 0.945 & 0.971 & 0.960 \\
                        & GraphGPS & 0.968 & \textbf{0.972} & 0.952 & 0.960 \\
                        & BiGCN & 0.951 & 0.923 & \underline{0.982} & 0.953 \\
                        & UPFD & 0.965 & 0.972 & 0.957 & 0.965 \\
                        & EBGCN & 0.964 & 0.966 & 0.962 & 0.963 \\
                        & GSAT & 0.956 & 0.949 & 0.962 & 0.955 \\
                        & GAMC & 0.946 & 0.941 & 0.946 & 0.943 \\
    \cmidrule(lr){1-6}
    \multirow{2}{*}{Ours} & \textbf{SDGL-DI} & \textbf{0.970} & \underline{0.969} & 0.969 & \textbf{0.970} \\
    & \textbf{SDGL-NC} & \underline{0.968} & \underline{0.969} & 0.966 & \underline{0.968} \\
    \bottomrule
  \end{tabular}
\end{table}

\subsection{Results for graph classification}
\label{sec:results-graph-classification}

We evaluate graph classification with long-range dependencies on the two FakeNewsNet subsets, \emph{PolitiFact} and \emph{GossipCop}, where each sample consists of a news item, its associated social-context graph, and a binary real/fake label.
Tables~\ref{ta_p_overreusult} and~\ref{tab:gossipcop} report the results on PolitiFact and GossipCop, respectively.
Overall, SDGL performs on both subsets, supporting our claim that stability-shaped deep propagation is beneficial when useful evidence is distributed over long propagation paths.

On \emph{PolitiFact}, SDGL-DI achieves the best overall performance, attaining the highest accuracy and F1 score among all compared methods, while SDGL-NC also remains competitive.
On \emph{GossipCop}, SDGL-DI again achieves the best accuracy and F1 score, and SDGL-NC matches the strongest baselines up to a very small margin.
Importantly, both SDGL variants outperform content-based baselines, which indicates that the social-context graph provides complementary information beyond text alone.
They also remain superior or highly competitive against strong propagation-structure-based methods, including GraphGPS, BiGCN, EBGCN, and UPFD.
These results suggest that the gain does not come merely from using graph propagation, but from controlling how deep propagation preserves informative variation over multiple hops.
In general, the results on PolitiFact and GossipCop demonstrate that SDGL is effective not only on node-level benchmarks but also on realistic graph-level tasks with long-range dependencies.

\subsection{Depth scaling with increasing layer number}
\label{sec:depth-scaling}

A central motivation of this work is to enable deep propagation without the severe performance degradation commonly observed in standard message-passing GNNs.
To evaluate this claim directly at the task level, we vary the number of layers and measure node classification accuracy on two representative benchmarks: Photo, a homophilic dataset and Chameleon, a heterophilic dataset.
For each depth, we report the mean and standard deviation over repeated runs.

\begin{figure}[t]
\begin{center}
\includegraphics[width=1.0\linewidth]{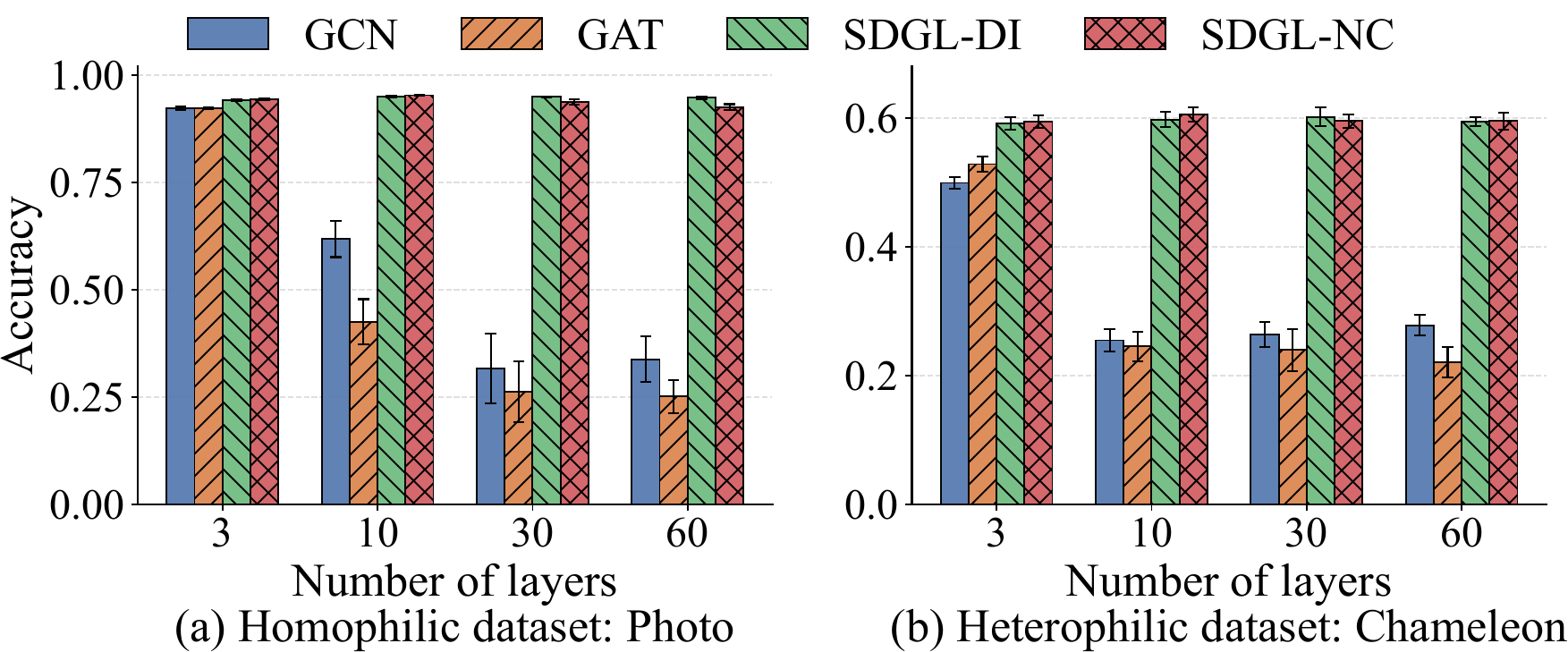}
\caption{Accuracy as a function of layer number on two representative node classification benchmarks.
\textbf{(a)} Homophilic dataset: Photo.
\textbf{(b)} Heterophilic dataset: Chameleon.
GCN and GAT degrade substantially as depth increases, whereas both SDGL variants remain stable and competitive even at 30--60 layers.
This supports our claim that stability shaping improves depth scaling under both homophilic and heterophilic graph conditions.}
\label{fig:acc_layers}
\end{center}
\vskip -0.1in
\end{figure}

Fig.~\ref{fig:acc_layers} shows a clear separation in depth behavior between standard diffusion-dominated baselines and our stability-shaped models.
On both datasets, GCN and GAT deteriorate rapidly as the number of layers increases, which is consistent with depth-induced degradation and over-smoothing.
In contrast, both SDGL-DI and SDGL-NC remain stable across depths and preserve strong predictive performance even in the 30- and 60-layer regimes. The effect is especially pronounced on the homophilic Photo benchmark.
While GCN and GAT suffer a severe drop in accuracy once the model becomes deep, both SDGL variants maintain near-saturated performance across all tested depths.
This indicates that stability shaping effectively prevents the strong collapse associated with repeated diffusion-style propagation.

A similar trend appears on the heterophilic Chameleon benchmark, where naive smoothing is often even more harmful because neighboring nodes may carry incompatible label information.
Again, the two SDGL variants substantially outperform GCN and GAT at larger depths and exhibit much smaller performance decay as the number of layers grows.
This suggests that explicit control of mode-wise stability improves not only representation preservation in homophilic graphs but also robustness to harmful mixing in heterophilic settings. Overall, these results provide direct task-level evidence that the proposed stability shaping translates into improved depth scaling.
They complement the main classification tables and are further supported by the mechanism-oriented diagnostics in the following subsections, including Dirichlet energy evolution and spectral validation of near-critical propagation.

\begin{figure}[th] 
\vskip -0.0in
\begin{center}
\includegraphics[width=0.9\linewidth]{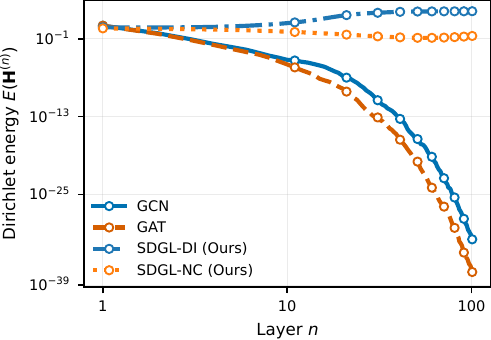}
\caption{Depth evolution of Dirichlet energy $E(\mathbf H^{(n)})$ in log scale from random initialization over 100 layers.
GCN and GAT rapidly drive $E(\mathbf H^{(n)})$ toward $0$, indicating over-smoothing, while our SDGL-NC keeps $E$ nearly constant and our SDGL-DI increases $E$ and then levels off, consistent with near critical propagation and band-limited amplification.}
\label{fig:dirichlet_energy}
\end{center}
\vskip -0.0in
\end{figure}
\subsection{Depth and evolution of Dirichlet energy}
\label{sec:dirichlet-energy}

We use Dirichlet energy to quantify how node representations evolve with depth and to connect the empirical behavior to our stability view of over-smoothing.
Diffusion-dominated propagation is expected to rapidly contract non-synchronous components and drive representations toward the synchronized manifold.
By design, our SDGL-DI targets band-limited amplification in Regime~II, while SDGL-NC targets stable near-critical propagation in Regime~III, so deep propagation does not collapse. 
We quantify smoothing by the normalized Dirichlet energy of node representations,
\begin{equation}
E(\mathbf H^{(n)})
:=
\frac12\sum_{(i,j)\in\mathcal E}\widetilde A_{ij}
\|\mathbf h_i^{(n)}-\mathbf h_j^{(n)}\|_2^2,
\label{eq:dirichlet-energy-main}
\end{equation}
where $\widetilde A_{ij}$ is the normalized edge weight associated with the diffusion backbone.
Equivalently, $E(\mathbf H^{(n)})=\mathrm{tr}((\mathbf H^{(n)})^\top\widetilde{\mathbf L}\mathbf H^{(n)})$ up to the same normalization.
Smaller energy indicates smoother representations, and decay toward zero indicates collapse toward synchrony.
Following GraphCON~\cite{rusch2022graph}, we evaluate untrained deep stacks on a $10\times10$ grid graph with random initialization for $L=100$ layers and record $E(\mathbf H^{(n)})$ for $n=1,\dots,L$.

Fig.~\ref{fig:dirichlet_energy} shows three Dirichlet energy evolution behaviors that match our theory.
(1) GCN and GAT exhibit a rapid decay of Dirichlet energy, indicating strong contraction of non-synchronous components and classical over-smoothing in Regime~I. 
(2) SDGL-NC keeps the energy close to its initial scale with only a mild decay, consistent with stable near-critical propagation in Regime~III. 
(3) SDGL-DI increases the energy and then saturates, consistent with diffusion-induced instability that amplifies a band of graph modes in Regime~II. 
Overall, these trajectories support stability shaping as a mechanism for preventing depth-induced collapse, either by preserving variation in SDGL-NC or by selectively amplifying modes in SDGL-DI.

\begin{figure}[th] 
\vskip -0.0in
\begin{center}
\includegraphics[width=1.0\linewidth]{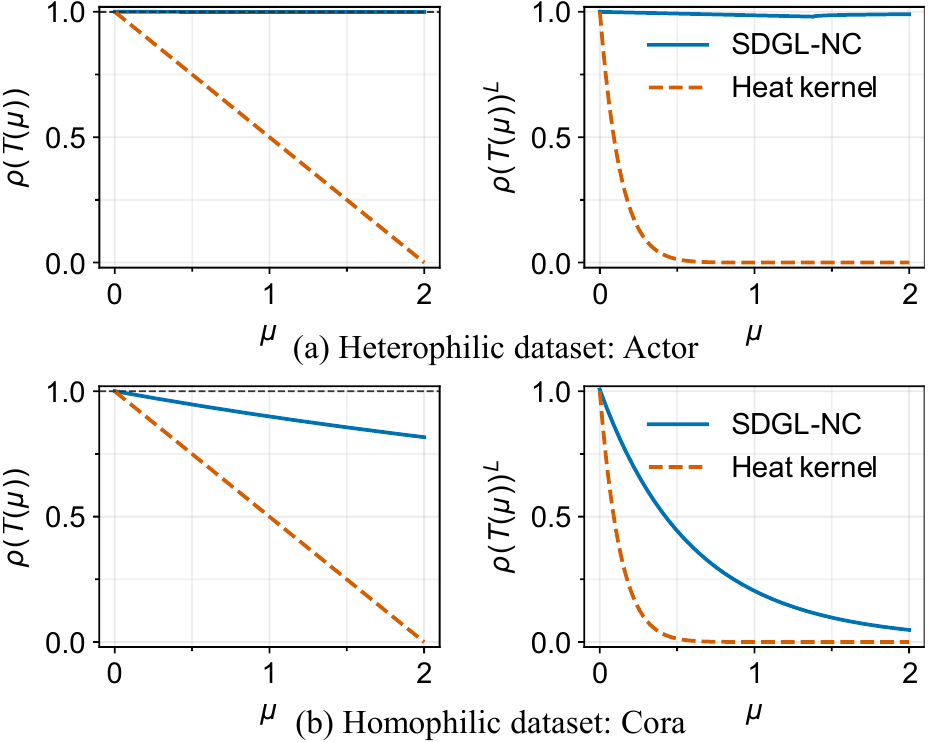 }
\caption{Validation of near critical propagation.
Spectral gain profiles of our SDGL-NC compared with a diffusion-only heat kernel on Actor and Cora.
Left panels plot the spectral radius $\rho(T(\mu))$, and right panels plot the depth-amplified gain $\rho(T(\mu))^{L}$ with $L=15$.
Our SDGL-NC stays closer to the critical boundary $\rho=1$ and retains substantially larger deep gain across frequencies, while the heat kernel rapidly suppresses nontrivial modes.}

\label{fig:nc_validation}
\end{center}
\vskip -0.2in
\end{figure}

\subsection{Validation of near-critical propagation}
\label{sec:val-near-critical}

We validate near-critical propagation by measuring the learned spectral gain of our SDGL-NC.
Following the linearized analysis, we construct the per-frequency transfer map $T(\mu)$ from the trained reaction and diffusion parameters and evaluate its spectral radius $\rho\!\left(T(\mu)\right)$ over $\mu\in[0,2]$.
To reflect depth effects, we also report the depth-amplified gain $\rho\!\left(T(\mu)\right)^{L}$ with $L$ equal to the network depth.
As a diffusion-only baseline, we include heat-kernel propagation.

Fig.~\ref{fig:nc_validation} shows a clear separation between SDGL-NC and the heat kernel on both Actor and Cora.
On Actor, $\rho\!\left(T(\mu)\right)$ stays close to $1$ across the spectrum. Hence $\rho\!\left(T(\mu)\right)^{L}$ remains near $1$ even at large depth.
This indicates stable but weakly contractive dynamics that preserve a broad range of spectral components rather than rapidly collapsing them.
In contrast, the heat kernel strongly attenuates mid- and high-frequency, and repeated propagation drives $\rho\!\left(T(\mu)\right)^{L}$ close to zero across most of the spectrum.

On Cora, SDGL-NC becomes more contractive as $\mu$ increases, yielding a stronger bias toward low-frequency components.
This matches the homophilic setting, where smoother signals tend to be more predictive, and aligns with prior findings on the utility of low-frequency information for homophilic graphs~\cite{Zhu2025SDMG}.
Importantly, SDGL-NC still maintains a substantially larger gain than the heat kernel across a wide range of $\mu$. Hence $\rho\!\left(T(\mu)\right)^{L}$ decays far more slowly, supporting deep propagation without over-smoothing.

\subsection{Parameter count scaling with depth}
\label{sec:param_count}

Although SDGL maintains two coupled fields and outputs a $2d$-dimensional representation, its propagation core remains parameter-efficient. In particular, each layer introduces only a small number of channel-wise reaction parameters together with a few diffusion scalars, whereas the dominant components of the model, namely the input lifting layers and the output classifier, are independent of depth. As a result, increasing depth has only a limited effect on the total parameter count. To quantify this effect, we compare the number of trainable parameters under varying depths $L\in\{3,10,30,60\}$ on the Photo dataset. For all models, we fix the hidden width to $128$ (and use $8$ attention heads for GAT) and vary only the number of propagation layers.

Table~\ref{tab:param_count_depth} shows that the parameter count of SDGL grows only mildly with depth. From $L=3$ to $L=60$, the total number of parameters increases from $0.196$M to $0.255$M for SDGL-NC and from $0.195$M to $0.239$M for SDGL-DI. In contrast, GCN and GAT exhibit much steeper growth because they require layer-specific weight matrices, and in the case of GAT, repeated multi-head projections. This difference becomes substantial in the deep regime. At $60$ layers, SDGL-NC and SDGL-DI use $4.14\times$ and $4.40\times$ fewer parameters than GCN, respectively, and $242\times$ and $258\times$ fewer parameters than GAT.

It is worth noting that at very small depth, such as $L=3$, GCN may have fewer parameters because SDGL incurs an upfront cost to lift the input into two interacting fields. However, this overhead is nearly depth-independent, and its relative impact quickly diminishes as the number of layers increases. These results indicate that SDGL supports substantially deeper propagation without the parameter growth typically associated with conventional deep message-passing architectures.

\begin{table}[t]
\centering
\caption{Trainable parameter counts (in millions) as a function of depth on Photo.
For all models, we fix the hidden width (and the number of attention heads for GAT) and vary only the number of propagation layers.}
\label{tab:param_count_depth}
\setlength{\tabcolsep}{6pt}
\begin{tabular}{rcccc}
\toprule
Layers & SDGL-NC (Ours) & SDGL-DI  (Ours)& GCN & GAT \\
\midrule
3  & 0.196 & 0.195 & 0.113 & 1.826 \\
10 & 0.203 & 0.201 & 0.229 & 9.187 \\
30 & 0.224 & 0.216 & 0.559 & 30.220 \\
60 & 0.255 & 0.239 & 1.054 & 61.770 \\
\bottomrule
\end{tabular}
\vskip -0.1in
\end{table}

\section{Conclusion}
We present SDGL, a stability-shaped approach to deep graph learning that turns over-smoothing into a mode-wise stability criterion around the synchronized manifold. By decoupling graph topology from local nonlinearities through a master-stability decomposition, we obtain a compact stability curve that predicts whether non-synchronous components will decay, persist, or grow with depth. This perspective yields two practical architectures with complementary behaviors: diffusion-induced band-limited amplification for pattern formation and stable near-critical propagation for deep models without collapse. Across homophilic and heterophilic node classification and long-range graph classification, SDGL consistently improves depth scaling and accuracy, and targeted diagnostics on Dirichlet energy, spectral gain, and grid dynamics align with the theory. These results suggest that stability shaping provides a principled and transferable approach for designing more efficient and accurate deep GNN propagation.

\ifCLASSOPTIONcaptionsoff
  \newpage
\fi

\bibliographystyle{IEEEtran}
\bibliography{Reference}

\newpage
\appendices
\section{Derivations}

\subsection{Derivation of the variational system in Equation~\eqref{eq:variational-main-rewrite}}
\label{app:msf-variational}
We derive the linearized variational system
\[
\frac{\mathrm d\,\Delta \mathbf H}{\mathrm dt}
=
\bigl(\mathbf I_n\otimes \mathbf J + \mathbf G\otimes \mathbf C\bigr)\,\Delta \mathbf H,
\qquad
\mathbf J:=\nabla f(\mathbf h^\star)\in\mathbb R^{q\times q},
\]
starting from the nodewise dynamics in Equation~\eqref{eq:rd-first}.

Recall the nodewise ODE
\begin{equation}
\frac{\mathrm d \mathbf h_i}{\mathrm dt}
=
f(\mathbf h_i)
+
\sum_{j=1}^n G_{ij}\,\mathbf C\,\mathbf h_j,
\qquad i=1,\dots,n,
\label{eq:node-ode-app}
\end{equation}
with node states $\mathbf h_i(t)\in\mathbb R^{q}$ and shared reaction $f\colon\mathbb R^{q}\to\mathbb R^{q}$.
Stack the node states and reactions as
\begin{equation}
\mathbf H
=
\begin{bmatrix}
\mathbf h_1\\[-1pt]
\vdots\\[-1pt]
\mathbf h_n
\end{bmatrix}
\in\mathbb R^{nq},
\qquad
\mathbf F(\mathbf H)
=
\begin{bmatrix}
f(\mathbf h_1)\\[-1pt]
\vdots\\[-1pt]
f(\mathbf h_n)
\end{bmatrix}
\in\mathbb R^{nq},
\end{equation}
and note that, by the Kronecker product identity,
\begin{equation}
(\mathbf G\otimes \mathbf C)\,\mathbf H
=
\begin{bmatrix}
\sum_{j} G_{1j}\,\mathbf C\,\mathbf h_j\\[-1pt]
\vdots\\[-1pt]
\sum_{j} G_{nj}\,\mathbf C\,\mathbf h_j
\end{bmatrix}.
\label{eq:GCY-block}
\end{equation}
The stacked dynamics can therefore be written as
\begin{equation}
\frac{\mathrm d \mathbf H}{\mathrm dt}
=
\mathbf F(\mathbf H)
+
(\mathbf G\otimes \mathbf C)\,\mathbf H .
\label{eq:stacked-ode-app}
\end{equation}

Assume there exists a nodewise equilibrium $\mathbf h^\star\in\mathbb R^{q}$ with
$f(\mathbf h^\star)=\mathbf 0$.
The corresponding synchronous state is
\begin{equation}
\mathbf H^\star = \mathbf 1\otimes \mathbf h^\star \in\mathbb R^{nq}.
\end{equation}
By definition of $\mathbf F$ we have $\mathbf F(\mathbf H^\star)=\mathbf 0$.
If moreover $\mathbf G\mathbf 1=\mathbf 0$ then
\begin{equation}
(\mathbf G\otimes \mathbf C)\,\mathbf H^\star
=
(\mathbf G\mathbf 1)\otimes(\mathbf C\mathbf h^\star)
=
\mathbf 0,
\end{equation}
so $\mathbf H^\star$ is an equilibrium of Equation~\eqref{eq:stacked-ode-app}.
Define perturbations around synchrony by 
$\Delta \mathbf H = \mathbf H - \mathbf H^\star$ .

We now linearize Equation~\eqref{eq:stacked-ode-app} around $\mathbf H^\star$:
\begin{align}
\frac{\mathrm d\,\Delta \mathbf H}{\mathrm dt}
&=
\frac{\mathrm d \mathbf H}{\mathrm dt} - \frac{\mathrm d \mathbf H^\star}{\mathrm dt}
\\
&=
\mathbf F(\mathbf H) + (\mathbf G\otimes \mathbf C)\,\mathbf H
-
\mathbf F(\mathbf H^\star) - (\mathbf G\otimes \mathbf C)\,\mathbf H^\star
\\
&=
\mathbf F(\mathbf H^\star + \Delta \mathbf H)
+
(\mathbf G\otimes \mathbf C)\,(\mathbf H^\star + \Delta \mathbf H)\\
&-
\mathbf F(\mathbf H^\star) - (\mathbf G\otimes \mathbf C)\,\mathbf H^\star
\\
&=
\mathbf F(\mathbf H^\star + \Delta \mathbf H)
-
\mathbf F(\mathbf H^\star)
+
(\mathbf G\otimes \mathbf C)\,\Delta \mathbf H
\\
&=
\mathbf F(\mathbf H^\star)
+
\nabla \mathbf F(\mathbf H^\star)\,\Delta \mathbf H
+
\mathcal O(\|\Delta \mathbf H\|^2) \\
&-
\mathbf F(\mathbf H^\star)
+
(\mathbf G\otimes \mathbf C)\,\Delta \mathbf H
\label{eq:deltaY-taylor-step}
\\
&=
\nabla \mathbf F(\mathbf H^\star)\,\Delta \mathbf H
+
\mathcal O(\|\Delta \mathbf H\|^2)
+
(\mathbf G\otimes \mathbf C)\,\Delta \mathbf H
\\
&\approx
\nabla \mathbf F(\mathbf H^\star)\,\Delta \mathbf H
+
(\mathbf G\otimes \mathbf C)\,\Delta \mathbf H
\\
&=
(\mathbf I_n\otimes \mathbf J)\,\Delta \mathbf H
+
(\mathbf G\otimes \mathbf C)\,\Delta \mathbf H
\\
&=
\bigl(\mathbf I_n\otimes \mathbf J + \mathbf G\otimes \mathbf C\bigr)\,\Delta \mathbf H .
\label{eq:deltaY-ode-start}
\end{align}
Here Equation~\eqref{eq:deltaY-taylor-step} uses a first-order Taylor expansion of $\mathbf F$ at $\mathbf H^\star$,
and we keep only the linear terms in $\Delta \mathbf H$.
Since $\mathbf F$ acts nodewise, its Jacobian at $\mathbf H^\star$ is block diagonal with
copies of the single-node Jacobian $\mathbf J=\nabla f(\mathbf h^\star)$, which gives
$\nabla \mathbf F(\mathbf H^\star)=\mathbf I_n\otimes \mathbf J$.
Equation~\eqref{eq:deltaY-ode-start} is exactly the variational system in
Equation~\eqref{eq:variational-main-rewrite}.

\subsection{Spectral Equivalence of Random-Walk and Symmetric Diffusion}
\label{app:lem:norm-vs-rw-main}
\begin{lemma}[Random-walk and symmetric diffusion are spectrally equivalent]
\label{lem:norm-vs-rw-main}
The symmetric normalized Laplacian $\widetilde{\mathbf L}_{\mathrm{sym}}$ and the random-walk Laplacian $\widetilde{\mathbf L}_{\mathrm{rw}}$ are similar and therefore share the same eigenvalues.
Under the node-wise change of variables $\mathbf z=\mathbf D^{-1/2}\mathbf r$, diffusion driven by $\widetilde{\mathbf L}_{\mathrm{sym}}$ is mapped to diffusion driven by $\widetilde{\mathbf L}_{\mathrm{rw}}$.
In the transformed coordinates, the synchronous manifold is uniform ($\mathbf z_i\equiv c$ for all $i$), and the master stability analysis depends on the same spectrum as in the original coordinates.
\end{lemma}

In practice, Lemma~\ref{lem:norm-vs-rw-main} means that choosing symmetric or random-walk normalization for the diffusion backbone does not affect our MSF-based conclusions: both normalizations induce the same diffusion eigenvalues, up to a degree reweighting of node features.

\begin{proof}
\textbf{(i) Similarity of $\widetilde{\mathbf L}_{\mathrm{sym}}$ and $\widetilde{\mathbf L}_{\mathrm{rw}}$.}
Recall
\(
\widetilde{\mathbf L}_{\mathrm{sym}}=\mathbf I-\mathbf D^{-1/2}\mathbf A\mathbf D^{-1/2}
\)
and
\(
\widetilde{\mathbf L}_{\mathrm{rw}}=\mathbf I-\mathbf D^{-1}\mathbf A.
\)
We compute
\begin{align}
\widetilde{\mathbf L}_{\mathrm{rw}}
&= \mathbf I - \mathbf D^{-1}\mathbf A \\
&= \mathbf I - \mathbf D^{-1/2}\bigl(\mathbf D^{-1/2}\mathbf A\mathbf D^{-1/2}\bigr)\mathbf D^{1/2} \\
&= \mathbf D^{-1/2}\Bigl(\mathbf I - \mathbf D^{-1/2}\mathbf A\mathbf D^{-1/2}\Bigr)\mathbf D^{1/2} \\
&= \mathbf D^{-1/2}\,\widetilde{\mathbf L}_{\mathrm{sym}}\,\mathbf D^{1/2}.
\end{align}
Equivalently,
\begin{equation}
\widetilde{\mathbf L}_{\mathrm{sym}}
= \mathbf D^{1/2}\,\widetilde{\mathbf L}_{\mathrm{rw}}\,\mathbf D^{-1/2},
\end{equation}
so $\widetilde{\mathbf L}_{\mathrm{sym}}$ and $\widetilde{\mathbf L}_{\mathrm{rw}}$ are similar and share identical eigenvalues.
In particular, if $\widetilde{\mathbf L}_{\mathrm{rw}}\mathbf v_k=\lambda_k\mathbf v_k$, then
$\widetilde{\mathbf L}_{\mathrm{sym}}(\mathbf D^{1/2}\mathbf v_k)=\lambda_k(\mathbf D^{1/2}\mathbf v_k)$.

\medskip
\noindent
\textbf{(ii) Change of variables for diffusion dynamics.}
For clarity, consider a scalar feature per node and write $\mathbf r(t)\in\mathbb R^n$ for the node field.
Diffusion driven by $\widetilde{\mathbf L}_{\mathrm{sym}}$ reads
\begin{equation}
\frac{\mathrm d \mathbf r(t)}{\mathrm dt} \;=\; -\widetilde{\mathbf L}_{\mathrm{sym}}\,\mathbf r(t).
\end{equation}
Define $\mathbf z(t)=\mathbf D^{-1/2}\mathbf r(t)$, i.e., $\mathbf r(t)=\mathbf D^{1/2}\mathbf z(t)$.
Then
\begin{align}
\frac{\mathrm d \mathbf z(t)}{\mathrm dt}
&= \mathbf D^{-1/2}\,\frac{\mathrm d \mathbf r(t)}{\mathrm dt} \\
&= -\mathbf D^{-1/2}\,\widetilde{\mathbf L}_{\mathrm{sym}}\,\mathbf r(t) \\
&= -\mathbf D^{-1/2}\,\widetilde{\mathbf L}_{\mathrm{sym}}\,\mathbf D^{1/2}\,\mathbf z(t) \\
&= -\widetilde{\mathbf L}_{\mathrm{rw}}\,\mathbf z(t),
\end{align}
where we used $\widetilde{\mathbf L}_{\mathrm{rw}}=\mathbf D^{-1/2}\widetilde{\mathbf L}_{\mathrm{sym}}\mathbf D^{1/2}$.
Thus diffusion in $\mathbf r$-coordinates is equivalent, under the invertible map $\mathbf z=\mathbf D^{-1/2}\mathbf r$, to diffusion in $\mathbf z$-coordinates driven by $\widetilde{\mathbf L}_{\mathrm{rw}}$.

For vector-valued node states $\mathbf h_i(t)\in\mathbb R^{q}$ with coupling matrix $\mathbf C\in\mathbb R^{q\times q}$, the global diffusion operator is $\widetilde{\mathbf L}\otimes \mathbf C$.
With $\mathbf H=(\mathbf h_1;\dots;\mathbf h_n)$ and the block-diagonal transform $\mathbf T=\mathbf D^{-1/2}\otimes \mathbf I_q$, the same calculation yields
\begin{align}
\frac{\mathrm d \mathbf Z}{\mathrm dt}
&= \mathbf T\,\frac{\mathrm d \mathbf H}{\mathrm dt} \\
&= -\mathbf T\,(\widetilde{\mathbf L}_{\mathrm{sym}}\otimes \mathbf C)\,\mathbf H \\
&= -(\mathbf D^{-1/2}\widetilde{\mathbf L}_{\mathrm{sym}}\mathbf D^{1/2}\otimes \mathbf C)\,\mathbf Z \\
&= -(\widetilde{\mathbf L}_{\mathrm{rw}}\otimes \mathbf C)\,\mathbf Z,
\end{align}
so the modal (MSF) analysis is unchanged up to this similarity transform.

\medskip
\noindent
\textbf{(iii) Synchronous manifold and spectrum.}
First, $\widetilde{\mathbf L}_{\mathrm{rw}}$ admits the uniform eigenvector $\mathbf 1$ with eigenvalue $0$:
\begin{align}
\widetilde{\mathbf L}_{\mathrm{rw}}\mathbf 1
&= (\mathbf I - \mathbf D^{-1}\mathbf A)\,\mathbf 1
= \mathbf 1 - \mathbf D^{-1}\mathbf A\mathbf 1
= \mathbf 1 - \mathbf D^{-1}\mathbf d
= \mathbf 0.
\end{align}
Hence, in the transformed coordinates, the synchronous manifold is
\(
\mathcal S_{\mathbf z}=\{\mathbf 1\otimes y:\ y\in\mathbb R^{q}\}.
\)

In the original coordinates, the corresponding zero-eigenvector of $\widetilde{\mathbf L}_{\mathrm{sym}}$ is degree-weighted:
\begin{align}
\widetilde{\mathbf L}_{\mathrm{sym}}\,\mathbf D^{1/2}\mathbf 1
&= \bigl(\mathbf I - \mathbf D^{-1/2}\mathbf A\mathbf D^{-1/2}\bigr)\mathbf D^{1/2}\mathbf 1\\
&= \mathbf D^{1/2}\mathbf 1 - \mathbf D^{-1/2}\mathbf A\mathbf D^{-1/2}\mathbf D^{1/2}\mathbf 1\\
&= \mathbf D^{1/2}\mathbf 1 - \mathbf D^{-1/2}\mathbf A\mathbf 1\\
&= \mathbf D^{1/2}\mathbf 1 - \mathbf D^{-1/2}\mathbf d\\
&= \mathbf D^{1/2}\mathbf 1 - \mathbf D^{-1/2}\mathbf D\mathbf 1\\
&= \mathbf D^{1/2}\mathbf 1 - \mathbf D^{1/2}\mathbf 1 = \mathbf 0.
\end{align}
Thus the neutral direction is degree-weighted in $\mathbf r$-coordinates ($\mathbf r_i\propto\sqrt{d_i}$) but uniform in $\mathbf z$-coordinates ($\mathbf z_i\equiv c$).

Since similarity preserves eigenvalues, $\widetilde{\mathbf L}_{\mathrm{sym}}$ and $\widetilde{\mathbf L}_{\mathrm{rw}}$ share the same spectrum $\{\lambda_k\}_{k=1}^n$.
Therefore any MSF quantity that depends on the diffusion spectrum (equivalently, on the graph frequencies $\{\mu_k\}$ used in the main text) is identical under either normalization.
\end{proof}

\subsection{Proof of Theorem~\ref{thm:general-msf}}
\label{app:mode-odes}

\begin{theorem*}[General master-stability decomposition at equilibrium]
Assume $\mathbf G$ is diagonalizable with eigen-decomposition
$\mathbf G=\mathbf U\boldsymbol\Lambda \mathbf U^{-1}$,
$\boldsymbol\Lambda=\mathrm{diag}(\kappa_1,\dots,\kappa_n)$,
and define $\mu_k:=-\kappa_k\ge 0$.
Let the modal coordinates be $\boldsymbol\xi=(\mathbf U^{-1}\otimes \mathbf I_q)\Delta\mathbf H$.
Then each mode $\boldsymbol\xi_k\in\mathbb R^q$ satisfies
\begin{equation}
\frac{\mathrm d \boldsymbol\xi_k}{\mathrm dt}
=
\bigl(\mathbf J-\mu_k \mathbf C\bigr)\boldsymbol\xi_k,
\qquad k=1,\dots,n.
\label{eq:app:mode-ode-mu}
\end{equation}
\end{theorem*}

\begin{proof}
Define the modal coordinates by
\begin{equation}
\Delta \mathbf H = (\mathbf U\otimes \mathbf I_q)\,\boldsymbol\xi,
\qquad
\boldsymbol\xi = (\mathbf U^{-1}\otimes \mathbf I_q)\,\Delta \mathbf H.
\label{eq:xi-def-app}
\end{equation}
Since $\mathbf U$ is constant in time, differentiating yields
\begin{equation}
\frac{\mathrm d \boldsymbol\xi}{\mathrm dt}
=
(\mathbf U^{-1}\otimes \mathbf I_q)\,\frac{\mathrm d\,\Delta \mathbf H}{\mathrm dt}.
\label{eq:xi-dot-def-app}
\end{equation}
Substitute Equation~\eqref{eq:variational-main-rewrite},
\(
\frac{\mathrm d\,\Delta \mathbf H}{\mathrm dt}
=
(\mathbf I_n\otimes \mathbf J + \mathbf G\otimes \mathbf C)\Delta\mathbf H,
\)
and use $\Delta\mathbf H=(\mathbf U\otimes \mathbf I_q)\boldsymbol\xi$:
\begin{align}
\frac{\mathrm d \boldsymbol\xi}{\mathrm dt}
&=
(\mathbf U^{-1}\otimes \mathbf I_q)\,\frac{\mathrm d\,\Delta \mathbf H}{\mathrm dt}
\label{eq:modal-step-0}
\\
&=
(\mathbf U^{-1}\otimes \mathbf I_q)\,
\bigl(\mathbf I_n\otimes \mathbf J + \mathbf G\otimes \mathbf C\bigr)\,
\Delta \mathbf H
\label{eq:modal-step-1}
\\
&=
(\mathbf U^{-1}\otimes \mathbf I_q)\,\bigl(\mathbf I_n\otimes \mathbf J\bigr)\,\Delta \mathbf H
+
(\mathbf U^{-1}\otimes \mathbf I_q)\,\bigl(\mathbf G\otimes \mathbf C\bigr)\,\Delta \mathbf H
\label{eq:modal-step-2}
\\
&=
(\mathbf U^{-1}\otimes \mathbf J)\,\Delta \mathbf H
+
(\mathbf U^{-1}\mathbf G\otimes \mathbf C)\,\Delta \mathbf H
\label{eq:modal-step-3}
\\
&=
(\mathbf U^{-1}\otimes \mathbf J)\,(\mathbf U\otimes \mathbf I_q)\,\boldsymbol\xi
+
(\mathbf U^{-1}\mathbf G\otimes \mathbf C)\,(\mathbf U\otimes \mathbf I_q)\,\boldsymbol\xi
\label{eq:modal-step-4}
\\
&=
(\mathbf U^{-1}\mathbf U\otimes \mathbf J)\,\boldsymbol\xi
+
(\mathbf U^{-1}\mathbf G\mathbf U\otimes \mathbf C)\,\boldsymbol\xi
\label{eq:modal-step-5}
\\
&=
(\mathbf I_n\otimes \mathbf J)\,\boldsymbol\xi
+
(\boldsymbol\Lambda\otimes \mathbf C)\,\boldsymbol\xi
\label{eq:modal-step-6}
\\
&=
\bigl(\mathbf I_n\otimes \mathbf J + \boldsymbol\Lambda\otimes \mathbf C\bigr)\,\boldsymbol\xi .
\label{eq:modal-step-7}
\end{align}
Since $\boldsymbol\Lambda=\mathrm{diag}(\kappa_1,\dots,\kappa_n)$, the block form is
\begin{equation}
\mathbf I_n\otimes \mathbf J + \boldsymbol\Lambda\otimes \mathbf C
=
\mathrm{diag}\bigl(\mathbf J+\kappa_1\mathbf C,\dots,\mathbf J+\kappa_n\mathbf C\bigr).
\end{equation}
Writing $\boldsymbol\xi=[\boldsymbol\xi_1;\dots;\boldsymbol\xi_n]$ with $\boldsymbol\xi_k\in\mathbb R^{q}$, each mode satisfies
\[
\frac{\mathrm d \boldsymbol\xi_k}{\mathrm dt}
=
(\mathbf J+\kappa_k\mathbf C)\,\boldsymbol\xi_k.
\]
Finally, substituting $\mu_k\coloneqq-\kappa_k\ge 0$ gives
\[
\frac{\mathrm d \boldsymbol\xi_k}{\mathrm dt}
=
(\mathbf J-\mu_k\mathbf C)\,\boldsymbol\xi_k,
\]
which is Equation~\eqref{eq:app:mode-ode-mu}.
\end{proof}

\subsection{Spectral properties of $\mathbf J-\mu \mathbf C$ and the NC mode map}
\label{app:disc-msf-details}

We collect spectral calculations underlying Section~\ref{sec:theory} and
Theorem~\ref{thm-near-unit-radius}.
For clarity, we work at the level of a single feature coordinate and then lift to the full Jacobian.

\subsubsection{Eigenvalues of the $2\times2$ NC blocks.}
Recall from the node-wise predator--prey (NC) reaction that for each feature index $j=1,\dots,d$
the coexistence equilibrium is
\[
r_j^\star=\gamma_j/\delta_j,
\qquad
f_j^\star=\alpha_j/\beta_j,
\]
and the corresponding $2\times2$ Jacobian block of $\mathbf J=\nabla f(\mathbf h^\star)$ is
\[
\mathbf J_j
=
\begin{bmatrix}
\alpha_j-\beta_j f_j^\star & -\beta_j r_j^\star\\[2pt]
\delta_j f_j^\star & \delta_j r_j^\star-\gamma_j
\end{bmatrix}
=
\begin{bmatrix}
0 & -\beta_j\gamma_j/\delta_j\\[2pt]
\delta_j\alpha_j/\beta_j & 0
\end{bmatrix}.
\]
The feature-space coupling matrix is block diagonal with
\[
\mathbf C_j
=
\begin{bmatrix}
C_R & 0\\[2pt]
0 & C_F
\end{bmatrix},
\qquad
C_R,C_F>0.
\]
For a graph mode with coupling strength $\mu\ge 0$, the $j$-th $2\times2$ block of
$\mathbf J-\mu\mathbf C$ is
\begin{equation}
\mathbf M_j(\mu)
\coloneqq
\mathbf J_j-\mu \mathbf C_j
=
\begin{bmatrix}
-\mu C_R & -\beta_j\gamma_j/\delta_j\\[2pt]
\delta_j\alpha_j/\beta_j & -\mu C_F
\end{bmatrix}.
\label{eq:Mj-def-app}
\end{equation}

Its trace and determinant are
\begin{align}
\operatorname{tr}\bigl(\mathbf M_j(\mu)\bigr)
&= -\mu C_R + (-\mu C_F)
\\
&= -\mu(C_R+C_F),
\label{eq:Mj-trace-app}
\\[4pt]
\det\bigl(\mathbf M_j(\mu)\bigr)
&= (-\mu C_R)(-\mu C_F)
\;-\;
\Bigl(-\tfrac{\beta_j\gamma_j}{\delta_j}\Bigr)
\Bigl(\tfrac{\delta_j\alpha_j}{\beta_j}\Bigr)
\\
&= \mu^2 C_R C_F + \alpha_j\gamma_j.
\label{eq:Mj-det-app}
\end{align}
Define $\omega_j:=\sqrt{\alpha_j\gamma_j}>0$.
Since $\mathbf M_j(\mu)$ has real entries, its two eigenvalues
$\lambda_{j,+}(\mu),\lambda_{j,-}(\mu)$ are either both real or a complex conjugate pair.
In either case
\begin{align}
&\lambda_{j,+}(\mu)+\lambda_{j,-}(\mu)
= \operatorname{tr}\bigl(\mathbf M_j(\mu)\bigr),
\\
&\lambda_{j,+}(\mu)\lambda_{j,-}(\mu)
= \det\bigl(\mathbf M_j(\mu)\bigr).
\end{align}
Writing $\lambda_{j,+}(\mu)=a+ib$ and $\lambda_{j,-}(\mu)=a-ib$ if they are non real, we obtain
from Equations~\eqref{eq:Mj-trace-app}–\eqref{eq:Mj-det-app}
\begin{align}
2a
&=
\lambda_{j,+}(\mu)+\lambda_{j,-}(\mu)
=
-\mu(C_R+C_F),
\\[2pt]
a^2+b^2
&=
\lambda_{j,+}(\mu)\lambda_{j,-}(\mu)
=
\mu^2 C_R C_F + \omega_j^2.
\end{align}
Thus for every $\mu>0$ the eigenvalues of $\mathbf M_j(\mu)$ satisfy
\begin{equation}
\Re\lambda_{j,\pm}(\mu)
=
-\frac{\mu(C_R+C_F)}{2}
<0,
\,\,
\bigl|\lambda_{j,\pm}(\mu)\bigr|^2
=
\mu^2 C_R C_F + \omega_j^2.
\label{eq:Mj-eigs-summary}
\end{equation}
In particular, all non synchronous modes are strictly stable in continuous time.

\subsubsection{Explicit Euler and the NC mode map.}
The per-mode explicit NC layer is
\begin{equation}
T_{\mathrm{NC}}(\mu)
=
\mathbf I_q+\Delta t\bigl(\mathbf J-\mu \mathbf C\bigr),
\end{equation}
If $\lambda$ is an eigenvalue of $\mathbf J-\mu \mathbf C$ then
\begin{equation}
\theta(\Delta t,\lambda)
=
1+\Delta t\,\lambda
\end{equation}
is an eigenvalue of $T_{\mathrm{NC}}(\mu)$.
Writing $\lambda=a+ib$ and using $|x+iy|^2=x^2+y^2$ gives
\begin{align}
\bigl|\theta(\Delta t,\lambda)\bigr|^2
&=
\bigl|1+\Delta t(a+ib)\bigr|^2
\\
&=
(1+\Delta t a)^2+(\Delta t b)^2
\\
&=
1+2\Delta t\,\Re\lambda+(\Delta t)^2|\lambda|^2.
\label{eq:theta-modulus}
\end{align}
The scalar factor $1+\Delta t\lambda$ is stable if and only if
$\bigl|\theta(\Delta t,\lambda)\bigr|<1$, which by
\eqref{eq:theta-modulus} is equivalent to
\begin{equation}
2\Delta t\,\Re\lambda+(\Delta t)^2|\lambda|^2<0.
\label{eq:euler-stab-cond-app}
\end{equation}

Substituting Equation \eqref{eq:Mj-eigs-summary} into
Equation \eqref{eq:euler-stab-cond-app} for $\lambda=\lambda_{j,\pm}(\mu)$ yields
\begin{align}
0
&>
2\Delta t\Bigl(-\frac{\mu(C_R+C_F)}{2}\Bigr)
+(\Delta t)^2\bigl(\mu^2 C_R C_F+\omega_j^2\bigr)
\\
&=
-\mu(C_R+C_F)\Delta t
+(\Delta t)^2\bigl(\mu^2 C_R C_F+\omega_j^2\bigr)
\\
&=
\Delta t
\Bigl(
-\mu(C_R+C_F)
+\Delta t\bigl(\mu^2 C_R C_F+\omega_j^2\bigr)
\Bigr).
\end{align}
For $\Delta t>0$ this holds if and only if
\begin{equation}
\Delta t
<
\Delta t_{j,\max}(\mu)
:=
\frac{\mu(C_R+C_F)}{\mu^2 C_R C_F+\omega_j^2}.
\label{eq:dt-max-single-app}
\end{equation}
For each non synchronous graph mode $k\ge2$ with coupling strength
$\mu_k>0$ and feature index $j$ this defines an upper bound
$\Delta t_{j,\max}(\mu_k)$.
A uniform stability threshold is obtained by
\begin{equation}
0<\Delta t<\Delta t_{\mathrm{stab}}
:=
\min_{k\ge2}\min_{j}\Delta t_{j,\max}(\mu_k),
\label{eq:dt-stab-app}
\end{equation}
which guarantees $\rho\bigl(T_{\mathrm{NC}}(\mu_k)\bigr)<1$ for every
non synchronous mode $k\ge2$.
\medskip

\subsection{Proof of Theorem~\ref{thm-near-unit-radius}}
\label{app:proof-near-unit-radius}
\begin{theorem*}[Spectral radius close to one for NC layers]
Assume positive predator--prey parameters, such that the node-wise reaction admits a coexistence equilibrium and the Jacobian $\mathbf J$ at this equilibrium has purely imaginary eigenvalues in each $2\times 2$ feature block.
Let $\mathbf C=\mathrm{diag}(C_R\mathbf I_d,\,C_F\mathbf I_d)$ with $C_R,C_F>0$.
Let $\{\mu_k\}_{k=1}^n$ be the nonnegative eigenvalues of $-\mathbf G$, ordered so that $\mu_1=0$ for a connected graph.
Then for every non-synchronous mode $k\ge 2$ and any tolerance $\varepsilon_{\text{MSF}}\in(0,1)$,
there exists $\Delta t_\varepsilon>0$ such that for all $0<\Delta t\le \Delta t_\varepsilon$,
\[
1-\varepsilon_{\text{MSF}}
\;\le\;
\rho\bigl(T_{\mathrm{NC}}(\mu_k)\bigr)
\;<\;
1,
\]
where $\rho(\cdot)$ denotes the spectral radius.
\end{theorem*}

\begin{proof}
We prove that for small enough step size $\Delta t$ the explicit NC layer has spectral
radius uniformly bounded in the interval $[1-\varepsilon,1)$ on all non synchronous
graph modes.

Let $\Lambda_s$ denote the finite set of eigenvalues of the continuous time mode matrices
$\mathbf J-\mu_k \mathbf C$ over all non synchronous modes $k\ge2$ and all feature
blocks:
\[
\Lambda_s
:=
\bigl\{\lambda_{j,\pm}(\mu_k)\colon k\ge2,\ j=1,\dots,d\bigr\}.
\]
By Equation \eqref{eq:Mj-eigs-summary} in Appendix~\ref{app:disc-msf-details} each $\lambda\in\Lambda_s$ satisfies $\Re\lambda<0$ and
$|\lambda|>0$.
For a fixed $\lambda\in\Lambda$ the corresponding eigenvalue of the NC transfer matrix
is
\[
\theta(\Delta t,\lambda)
=
1+\Delta t\,\lambda,
\]
and by Equation \eqref{eq:theta-modulus} the squared modulus is
\[
\bigl|\theta(\Delta t,\lambda)\bigr|^2
=
1+2\Delta t\,\Re\lambda+(\Delta t)^2|\lambda|^2.
\]

\textbf{Continuity and stability.}
The map $\Delta t\mapsto\theta(\Delta t,\lambda)$ is continuous and
$\theta(0,\lambda)=1$.
Hence
\[
\lim_{\Delta t\to0}\bigl|\theta(\Delta t,\lambda)\bigr|=1
\qquad
\text{for every }\lambda\in\Lambda_s.
\]
Moreover, for each fixed $\lambda\in\Lambda_s$ the explicit Euler stability condition Equation
\eqref{eq:euler-stab-cond-app} gives a mode dependent bound
$\Delta t<\Delta t_{\max}(\lambda)$ such that
$\bigl|\theta(\Delta t,\lambda)\bigr|<1$ whenever
$0<\Delta t<\Delta t_{\max}(\lambda)$.
By Equation \eqref{eq:dt-stab-app} there exists a uniform stability threshold
$\Delta t_{\mathrm{stab}}>0$ so that, for all $\lambda\in\Lambda_s$, 
\begin{equation}
0<\Delta t<\Delta t_{\mathrm{stab}}
\quad\Longrightarrow\quad
\bigl|\theta(\Delta t,\lambda)\bigr|<1.
\label{eq:uniform-stability-app}
\end{equation}

\textbf{Uniform near unit radius.}
Fix $\varepsilon_{\text{MSF}}\in(0,1)$.
Since $\Lambda$ is finite and
$\bigl|\theta(\Delta t,\lambda)\bigr|\to1$ as $\Delta t\to0$ for every
$\lambda\in\Lambda_s$, there exists $\delta_{\text{MSF}}>0$ such that, for all $\lambda\in\Lambda_s$, 
\begin{equation}
0<\Delta t<\delta_{\text{MSF}}
\Longrightarrow
1-\varepsilon_{\text{MSF}}
<
\bigl|\theta(\Delta t,\lambda)\bigr|
<1.
\label{eq:uniform-band-app}
\end{equation}
Define
\[
\Delta t_\varepsilon
:=
\min\{\delta_{\text{MSF}},\Delta t_{\mathrm{stab}}\}.
\]
For any $0<\Delta t\le\Delta t_\varepsilon$ both
Equations \eqref{eq:uniform-stability-app} and \eqref{eq:uniform-band-app} hold.
For a given non synchronous graph mode $k\ge2$ the eigenvalues of
$T_{\mathrm{NC}}(\mu_k)$ are precisely the numbers
$\theta(\Delta t,\lambda)$ with $\lambda\in\Lambda^k_s\subset\Lambda_s$, so
\begin{align}
\rho\bigl(T_{\mathrm{NC}}(\mu_k)\bigr)
&=
\max_{\lambda\in\Lambda^k_s}\bigl|\theta(\Delta t,\lambda)\bigr| <1
\end{align}
by the upper bound in Equation \eqref{eq:uniform-band-app}, and
\begin{align}
\rho\bigl(T_{\mathrm{NC}}(\mu_k)\bigr)
&=
\max_{\lambda\in\Lambda^k_s}\bigl|\theta(\Delta t,\lambda)\bigr| \ge 1-\varepsilon_{\text{MSF}}
\end{align}
by the lower bound in Equation \eqref{eq:uniform-band-app}.
Thus, for all non synchronous modes  $k\ge2$
\[
1-\varepsilon_{\text{MSF}}
\;\le\;
\rho\bigl(T_{\mathrm{NC}}(\mu_k)\bigr)
\;<\;
1
\]
whenever $0<\Delta t\le\Delta t_\varepsilon$, which proves
Theorem~\ref{thm-near-unit-radius}.
\end{proof}
\qedhere

\section{Additional Derivations for Architectural Design}
\label{app:design}

\subsection{From continuous-time two-field RD to the general SDGL layer}
\label{app:rd-discretization}

We show how the general SDGL layer in
Equations~\eqref{eq:gen_update_R}--\eqref{eq:gen_update_F}
arises from a two-field reaction--diffusion (RD) template.

\subsubsection{Continuous-time two-field RD on a graph.}
Let $\mathbf r_i(t),\mathbf f_i(t)\in\mathbb R^{d}$ denote the two fields at node $i$.
We consider the node-wise RD template
\begin{align}
\frac{\mathrm d\mathbf r_i(t)}{\mathrm dt}
&=
f_R\bigl(\mathbf r_i(t),\mathbf f_i(t);\boldsymbol\theta\bigr)
+
C_R\sum_{j=1}^n G_{ij}\,\mathbf r_j(t),
\label{eq:app:rd-r-node}\\
\frac{\mathrm d\mathbf f_i(t)}{\mathrm dt}
&=
f_F\bigl(\mathbf r_i(t),\mathbf f_i(t);\boldsymbol\theta\bigr)
+
C_F\sum_{j=1}^n G_{ij}\,\mathbf f_j(t),
\label{eq:app:rd-f-node}
\end{align}
where $f_R,f_F:\mathbb R^{d}\times\mathbb R^{d}\to\mathbb R^{d}$ are coordinate-wise reactions
(shared across nodes) with learnable parameters $\boldsymbol\theta$, and $C_R,C_F>0$ are
learnable diffusion strengths.
We use $\mathbf G=\widetilde{\mathbf S}-\mathbf I_n$ for a fixed backbone $\widetilde{\mathbf S}\in\mathbb R^{n\times n}$.

\subsubsection{Matrix form.}
Stack node-wise states row-wise as $\mathbf R(t),\mathbf F(t)\in\mathbb R^{n\times d}$, where
the $i$-th rows are $\mathbf r_i(t)^\top$ and $\mathbf f_i(t)^\top$.
Define the row-wise reaction maps
$f_{\mathbf R}, f_{\mathbf F}:\mathbb R^{n\times d}\times\mathbb R^{n\times d}\to\mathbb R^{n\times d}$ by
\begin{align}
&\bigl[f_{\mathbf R}(\mathbf R,\mathbf F;\boldsymbol\theta)\bigr]_{i,:}
=
f_R(\mathbf r_i,\mathbf f_i;\boldsymbol\theta)^\top,
\\
&\bigl[f_{\mathbf F}(\mathbf R,\mathbf F;\boldsymbol\theta)\bigr]_{i,:}
=
f_F(\mathbf r_i,\mathbf f_i;\boldsymbol\theta)^\top.
\end{align}

Then Equations~\eqref{eq:app:rd-r-node}--\eqref{eq:app:rd-f-node} become
\begin{align}
\frac{\mathrm d\mathbf R}{\mathrm dt}
&=
f_{\mathbf R}(\mathbf R,\mathbf F;\boldsymbol\theta)
+ C_R\,\mathbf G\mathbf R,
\label{eq:app:rd-r-matrix}\\
\frac{\mathrm d\mathbf F}{\mathrm dt}
&=
f_{\mathbf F}(\mathbf R,\mathbf F;\boldsymbol\theta)
+ C_F\,\mathbf G\mathbf F.
\label{eq:app:rd-f-matrix}
\end{align}

\subsubsection{Explicit Euler discretization.}
Discretize time as $t_\ell=\ell\Delta t$ with step size $\Delta t>0$ and set
$\mathbf R^\ell\approx\mathbf R(t_\ell)$, $\mathbf F^\ell\approx\mathbf F(t_\ell)$.
One explicit Euler step applied to Equations~\eqref{eq:app:rd-r-matrix}--\eqref{eq:app:rd-f-matrix} yields
\begin{align}
\mathbf R^{\ell+1}
&=
\mathbf R^\ell
+ \Delta t\Big(f_{\mathbf R}(\mathbf R^\ell,\mathbf F^\ell;\boldsymbol\theta)+C_R\,\mathbf G\mathbf R^\ell\Big),
\label{eq:app:euler-R}\\
\mathbf F^{\ell+1}
&=
\mathbf F^\ell
+ \Delta t\Big(f_{\mathbf F}(\mathbf R^\ell,\mathbf F^\ell;\boldsymbol\theta)+C_F\,\mathbf G\mathbf F^\ell\Big),
\label{eq:app:euler-F}
\end{align}
which is exactly the general SDGL layer form in
Equations~\eqref{eq:gen_update_R}--\eqref{eq:gen_update_F}.

\subsection{SDGL-DI details: Schnakenberg reaction and shifted equilibrium}
\label{app:di-schnakenberg}

In this section, we provide the continuous-time Schnakenberg-type reaction used by SDGL-DI,
the equilibrium shift that yields $\mathbf R=\mathbf 0,\mathbf F=\mathbf 0$,
and the resulting discrete layer.

\subsubsection{Schnakenberg reaction (coordinate-wise).}
For each node $i$ and each coordinate $j\in\{1,\dots,d\}$, introduce physical variables
$P_{ij}(t)$ and $V_{ij}(t)$ (activator/inhibitor) and define the Schnakenberg reaction
\begin{align}
\frac{\mathrm d P_{ij}(t)}{\mathrm dt}
&=
\psi_j - P_{ij}(t) + \eta\,P_{ij}(t)^2 V_{ij}(t),
\label{eq:app:schnak-P-scalar}\\
\frac{\mathrm d V_{ij}(t)}{\mathrm dt}
&=
\chi_j - \eta\,P_{ij}(t)^2 V_{ij}(t),
\label{eq:app:schnak-V-scalar}
\end{align}
with learnable parameters $\boldsymbol\psi,\boldsymbol\chi\in\mathbb R^d_{>0}$ and $\eta\in\mathbb R_{>0}$.

In vector form (elementwise operations), at node $i$:
\begin{align}
\frac{\mathrm d\mathbf P_i(t)}{\mathrm dt}
&=
\boldsymbol\psi-\mathbf P_i(t)+\eta\bigl(\mathbf P_i(t)^{\odot 2}\odot \mathbf V_i(t)\bigr),
\label{eq:app:schnak-P-vec}\\
\frac{\mathrm d\mathbf V_i(t)}{\mathrm dt}
&=
\boldsymbol\chi-\eta\bigl(\mathbf P_i(t)^{\odot 2}\odot \mathbf V_i(t)\bigr),
\label{eq:app:schnak-V-vec}
\end{align}
where $\odot$ denotes the Hadamard product and $\mathbf u^{\odot 2}\coloneqq \mathbf u\odot\mathbf u$.

\subsubsection{Closed-form coexistence equilibrium.}
The node-wise equilibrium $(\mathbf P^\star,\mathbf V^\star)$ solves
\[
\mathbf 0=\boldsymbol\psi-\mathbf P^\star+\eta\bigl((\mathbf P^\star)^{\odot 2}\odot\mathbf V^\star\bigr),
\quad
\mathbf 0=\boldsymbol\chi-\eta\bigl((\mathbf P^\star)^{\odot 2}\odot\mathbf V^\star\bigr),
\]
which implies
\begin{equation}
\mathbf P^\star=\mathbf p^\star\coloneqq\boldsymbol\psi+\boldsymbol\chi,
\qquad
\mathbf V^\star=\mathbf v^\star\coloneqq\boldsymbol\chi\oslash\big(\eta(\mathbf p^\star)^{\odot 2}\big),
\label{eq:app:schnak-eq}
\end{equation}
where $\oslash$ denotes elementwise division. By construction,
$\eta(\mathbf p^\star)^{\odot 2}\odot\mathbf v^\star=\boldsymbol\chi$.

\subsubsection{Shifted variables and diffusion on the graph.}
We re-parameterize around equilibrium:
\begin{equation}
\mathbf P_i(t)=\mathbf r_i(t)+\mathbf p^\star,
\qquad
\mathbf V_i(t)=\mathbf f_i(t)+\mathbf v^\star.
\label{eq:app:shift}
\end{equation}
Augmenting Equations~\eqref{eq:app:schnak-P-vec}--\eqref{eq:app:schnak-V-vec} with diffusion on $\mathbf r_i,\mathbf f_i$
(using $\mathbf G=\widetilde{\mathbf S}-\mathbf I_n$) yields the SDGL-DI continuous-time system
\begin{align}
\frac{\mathrm d\mathbf r_i(t)}{\mathrm dt}
&=
\boldsymbol\psi-\mathbf P_i(t)+\eta\bigl(\mathbf P_i(t)^{\odot 2}\odot\mathbf V_i(t)\bigr)
+
C_R\sum_{j=1}^n G_{ij}\mathbf r_j(t),
\label{eq:app:di-cont-r}\\
\frac{\mathrm d\mathbf f_i(t)}{\mathrm dt}
&=
\boldsymbol\chi-\eta\bigl(\mathbf P_i(t)^{\odot 2}\odot\mathbf V_i(t)\bigr)
+
C_F\sum_{j=1}^n G_{ij}\mathbf f_j(t).
\label{eq:app:di-cont-f}
\end{align}
Setting $\mathbf r_i(t)=\mathbf 0$ and $\mathbf f_i(t)=\mathbf 0$ implies $\mathbf P_i(t)=\mathbf p^\star$ and $\mathbf V_i(t)=\mathbf v^\star$.
Then the reaction terms vanish by Equation~\eqref{eq:app:schnak-eq}, and the diffusion terms vanish trivially.
Hence $(\mathbf r_i,\mathbf f_i)=(\mathbf 0,\mathbf 0)$ is a node-wise equilibrium, and the synchronous equilibrium in the shifted coordinates is $\mathbf R=\mathbf 0,\mathbf F=\mathbf 0$.

\subsubsection{Discrete SDGL-DI layer.}
Applying explicit Euler with step size $\Delta t$ to Equations~\eqref{eq:app:di-cont-r}--\eqref{eq:app:di-cont-f} and defining
\[
\mathbf P^\ell=\mathbf R^\ell+\mathbf p^\star,\quad
\mathbf V^\ell=\mathbf F^\ell+\mathbf v^\star,\quad
\boldsymbol\zeta^\ell=\eta\bigl((\mathbf P^\ell)^{\odot 2}\odot\mathbf V^\ell\bigr)
\]
recovers Equations~\eqref{eq:sdgl-di-layer-r}--\eqref{eq:sdgl-di-layer-f} in the main text.

\subsection{SDGL-DI details: closed-form MSF instability band}
\label{app:di-msf-band}

We derive a verifiable diffusion-induced instability criterion for SDGL-DI and introduce a
graph-free training regularizer that depends only on reaction/diffusion parameters.

\begin{corollary}[Verifiable unstable band for $2\times2$ activator--inhibitor blocks]
\label{cor:turing-2x2-band}
Specialize to $q=2$ with anisotropic diffusion
$\mathbf C=\mathrm{diag}(C_R,C_F)$, where $C_R,C_F>0$.
Let $\mathbf J=\begin{bmatrix}a&b\\ c&e\end{bmatrix}$ and define
$\mathbf M(\mu)=\mathbf J-\mu\mathbf C$ for $\mu\ge0$.
If $\mathrm{tr}(\mathbf J)=a+e<0$, then $\mathrm{tr}(\mathbf M(\mu))<0$ for all $\mu>0$ and
a mode is unstable if and only if $\det(\mathbf M(\mu))<0$.
Moreover,
\begin{equation}
\det(\mathbf M(\mu))
=
\det(\mathbf J)
-\mu\,(aC_F+eC_R)
+\mu^2 C_RC_F.
\label{eq:det-quadratic}
\end{equation}
An unstable band exists if and only if
\[
\det(\mathbf J)>0,\,
B\coloneqq aC_F+eC_R>0,\,
B^2>4\,\det(\mathbf J)\,C_RC_F.
\]
In this case, $\det(\mathbf M(\mu))<0$ holds for $\mu\in(\mu_-,\mu_+)$, where $\mu_\pm$ are the two positive roots of Equation~\eqref{eq:det-quadratic}.
\end{corollary}

\subsubsection{Per-coordinate Jacobian at the shifted equilibrium.}
Since the Schnakenberg reaction is coordinate-wise, the Jacobian at the shifted equilibrium
$\mathbf r=\mathbf 0,\mathbf f=\mathbf 0$ is block diagonal with $2\times2$ blocks.
Fix a coordinate $j$ and write $p_j^\star=\psi_j+\chi_j$.
Differentiating the reaction terms in Equations~\eqref{eq:app:di-cont-r}--\eqref{eq:app:di-cont-f} at equilibrium yields
\begin{equation}
\mathbf J_j
=
\begin{bmatrix}
a_j & b_j\\
c_j & e_j
\end{bmatrix},
\label{eq:app:di-Jj}
\end{equation}
where
\begin{align}
a_j &= -1+\frac{2\chi_j}{p_j^\star}, &
b_j &= \eta(p_j^\star)^2, \notag\\
c_j &= -\frac{2\chi_j}{p_j^\star}, &
e_j &= -\eta(p_j^\star)^2. \notag
\end{align}
The feature-space diffusion coupling is $\mathbf C_j=\mathrm{diag}(C_R,C_F)$ with $C_R,C_F>0$.

\subsubsection{Mode matrix and determinant test.}
For a graph frequency $\mu\ge0$, the $j$-th mode matrix is
\[
\mathbf M_j(\mu)=\mathbf J_j-\mu\mathbf C_j.
\]
Its trace satisfies $\mathrm{tr}(\mathbf M_j(\mu))=\mathrm{tr}(\mathbf J_j)-\mu(C_R+C_F)$, so if $\mathrm{tr}(\mathbf J_j)<0$ then
$\mathrm{tr}(\mathbf M_j(\mu))<0$ for all $\mu>0$. In this case, instability occurs if and only if
$\det(\mathbf M_j(\mu))<0$. Expanding the determinant yields the quadratic form
\begin{equation}
\det(\mathbf M_j(\mu))
=
\det(\mathbf J_j)
-\mu\,B_j
+\mu^2 C_R C_F,
\label{eq:app:di-det-quad}
\end{equation}
where $B_j\coloneqq a_j C_F+e_j C_R$.
Define the discriminant
\[
\Delta_j\coloneqq B_j^2-4\,\det(\mathbf J_j)\,C_R C_F.
\]
By Corollary~\ref{cor:turing-2x2-band}, an unstable band exists for coordinate $j$ if and only if
\begin{equation}
\det(\mathbf J_j)>0,\qquad
B_j>0,\qquad
\Delta_j>0,
\label{eq:app:di-band-conds}
\end{equation}
in which case $\det(\mathbf M_j(\mu))<0$ holds exactly for $\mu\in(\mu_{-,j},\mu_{+,j})$, where
\begin{equation}
\mu_{\pm,j}=\frac{B_j\pm\sqrt{\Delta_j}}{2C_R C_F}.
\label{eq:app:di-mu-pm}
\end{equation}

Since normalized diffusion backbones have bounded spectra (e.g., $\mu\in[0,2]$), we optionally add a lightweight, graph-independent penalty on the reaction--diffusion parameters to encourage stable local dynamics and the presence of an unstable window.

\subsection{SDGL-NC details: coexistence equilibrium and purely imaginary local spectrum}
\label{app:nc-imag-spectrum}

We derive the coexistence equilibrium and the local Jacobian spectrum of the predator--prey reaction used in SDGL-NC.

\subsubsection{Predator--prey reaction (coordinate-wise).}
Ignoring diffusion, SDGL-NC uses the coordinate-wise reaction
\begin{align}
\frac{\mathrm d \mathbf r_i}{\mathrm dt}
&=
\boldsymbol\alpha\odot\mathbf r_i
-\boldsymbol\beta\odot\mathbf r_i\odot\mathbf f_i,
\label{eq:app:nc-r}\\
\frac{\mathrm d \mathbf f_i}{\mathrm dt}
&=
\boldsymbol\delta\odot\mathbf r_i\odot\mathbf f_i
-\boldsymbol\gamma\odot\mathbf f_i,
\label{eq:app:nc-f}
\end{align}
with learnable parameters
$\boldsymbol\alpha,\boldsymbol\beta,\boldsymbol\delta,\boldsymbol\gamma\in\mathbb R^d_{>0}$.

\textbf{Coexistence equilibrium.}
For each coordinate $j$, the nontrivial equilibrium satisfies
\[
0=r^\star_j(\alpha_j-\beta_j f^\star_j),\qquad
0=f^\star_j(\delta_j r^\star_j-\gamma_j),
\]
hence
\begin{equation}
r^\star_j=\frac{\gamma_j}{\delta_j},
\qquad
f^\star_j=\frac{\alpha_j}{\beta_j}.
\label{eq:app:nc-eq}
\end{equation}

\textbf{Jacobian and purely imaginary eigenvalues.}
At $(r^\star_j,f^\star_j)$, the $2\times2$ Jacobian block is
\begin{equation}
\mathbf J_j
=
\begin{bmatrix}
\frac{\partial \dot r_j}{\partial r_j} & \frac{\partial \dot r_j}{\partial f_j}\\[2pt]
\frac{\partial \dot f_j}{\partial r_j} & \frac{\partial \dot f_j}{\partial f_j}
\end{bmatrix}_{(r^\star_j,f^\star_j)}.
\label{eq:app:nc-Jj-def}
\end{equation}
Evaluating the derivatives at the coexistence equilibrium gives
\begin{equation}
\mathbf J_j
=
\begin{bmatrix}
0 & -\beta_j r^\star_j\\
\delta_j f^\star_j & 0
\end{bmatrix}
=
\begin{bmatrix}
0 & -\beta_j\gamma_j/\delta_j\\
\delta_j\alpha_j/\beta_j & 0
\end{bmatrix}.
\label{eq:app:nc-Jj}
\end{equation}
Its eigenvalues satisfy $\lambda^2+\alpha_j\gamma_j=0$, hence
\begin{equation}
\lambda_{j,\pm}=\pm i\sqrt{\alpha_j\gamma_j}.
\label{eq:app:nc-imag}
\end{equation}
Therefore, the local linearization has zero real part at the coexistence equilibrium, providing a near-critical starting point that diffusion then shifts mode-wise via $\mathbf J-\mu\mathbf C$.
\subsection{IMEX--Euler as an alternative discretization}
\label{app:imex_derivation}

Our framework is not tied to explicit Euler.
To illustrate that alternative time-stepping schemes can be plugged in seamlessly,
we derive an IMEX--Euler discretization for the SDGL-NC dynamics, treating the nonlinear reaction explicitly
and the diffusion term implicitly.

\textbf{Continuous-time SDGL-NC (matrix form).}
Consider the two-field ODEs
\begin{align}
\frac{\mathrm d\mathbf R}{\mathrm dt}
&= \boldsymbol\alpha\odot\mathbf R - \boldsymbol\beta\odot\mathbf R\odot\mathbf F
\;-\; C_R\,\widetilde{\mathbf L}\,\mathbf R,
\label{eq:imex:ode-r}\\
\frac{\mathrm d\mathbf F}{\mathrm dt}
&= \boldsymbol\delta\odot\mathbf R\odot\mathbf F - \boldsymbol\gamma\odot\mathbf F
\;-\; C_F\,\widetilde{\mathbf L}\,\mathbf F,
\label{eq:imex:ode-f}
\end{align}
where $\widetilde{\mathbf L}=\mathbf I-\widetilde{\mathbf S}$ and $\mathbf R(t),\mathbf F(t)\in\mathbb R^{n\times d}$
stack node states row-wise. Discretize time as $t_\ell=\ell\Delta t$ with $\Delta t>0$.

\textbf{IMEX--Euler update.}
Applying IMEX--Euler to Equation~\eqref{eq:imex:ode-r} gives
\begin{equation}
\mathbf R^{\ell+1}
= \mathbf R^\ell
+ \Delta t\Bigl(\boldsymbol\alpha\odot\mathbf R^\ell - \boldsymbol\beta\odot\mathbf R^\ell\odot\mathbf F^\ell\Bigr)
- \Delta t\,C_R\,\widetilde{\mathbf L}\,\mathbf R^{\ell+1},
\end{equation}
hence
\begin{equation}
\bigl(\mathbf I+\Delta t\,C_R\,\widetilde{\mathbf L}\bigr)\mathbf R^{\ell+1}
=
\mathbf R^\ell
+ \Delta t\Bigl(\boldsymbol\alpha\odot\mathbf R^\ell - \boldsymbol\beta\odot\mathbf R^\ell\odot\mathbf F^\ell\Bigr).
\label{eq:imex-r}
\end{equation}
Similarly, applying IMEX--Euler to \eqref{eq:imex:ode-f} yields
\begin{equation}
\bigl(\mathbf I+\Delta t\,C_F\,\widetilde{\mathbf L}\bigr)\mathbf F^{\ell+1}
=
\mathbf F^\ell
+ \Delta t\Bigl(\boldsymbol\delta\odot\mathbf R^\ell\odot\mathbf F^\ell - \boldsymbol\gamma\odot\mathbf F^\ell\Bigr).
\label{eq:imex-f}
\end{equation}
For the normalized Laplacian-type operators used in this work, $\widetilde{\mathbf L}$ has nonnegative spectrum,
so $\mathbf I+\Delta t\,C_\cdot\,\widetilde{\mathbf L}$ is invertible for $\Delta t>0$ and $C_\cdot>0$.

\section{Additional Experimental Details}
\label{app:exp-details}

\subsection{Dataset statistics}
\label{app:dataset-stats}

We evaluate SDGL on two tasks.
\textbf{(i) Transductive node classification.}
We use widely adopted single-graph benchmarks, including homophilic citation and co-purchase/co-authorship networks
(Cora, CiteSeer, PubMed, Amazon-Computers, Amazon-Photo, Coauthor-CS, Coauthor-Physics)
and heterophilic web/actor graphs (Actor, Chameleon, Squirrel).
For these datasets, we follow the standard public splits and protocols used in prior work.

\textbf{(ii) Graph classification via propagation structure.}
We use FakeNewsNet with its two standard subsets, \emph{PolitiFact} and \emph{GossipCop}.
Each instance corresponds to one news article paired with a diffusion graph derived from its social context,
including related Twitter posts and user engagements/interactions.
Nodes represent users/posts (depending on preprocessing), edges encode observed interaction links,
and the label indicates whether the article is \texttt{Real} or \texttt{Fake} based on expert verification.
We preprocess the data following~\cite{dou2021user} and use a 75\%/25\% train--test split.

\subsection{Hyperparameter Disclosure}
\label{sec:hyperparameter_disclosure}

Table~\ref{tab:sdgl_hparams_all} reports the hyperparameters used for SDGL-DI and SDGL-NC on all datasets.
For both variants, the model maintains two interacting fields with width $d$, so the concatenated representation has width $2d$.
Importantly, standard message-passing GNNs such as GCN are typically used with only $2$--$4$ layers in practice, whereas SDGL remains stable and effective at substantially larger depths, as reflected by the tuned layer numbers $L$ in Table~\ref{tab:sdgl_hparams_all}.

\begin{table*}[ht]
\centering
\caption{Hyperparameter disclosure for SDGL across datasets.
$L$ is the number of layers, $d$ is the hidden width per field (so the output width is $2d$), $p$ is the dropout rate, and $\mathrm{lr}$ is the learning rate.}
\label{tab:sdgl_hparams_all}
\setlength{\tabcolsep}{4pt}
\renewcommand{\arraystretch}{1.15}
\begin{tabular}{lcccccccc}
\toprule
& \multicolumn{4}{c}{\textbf{SDGL-DI}} & \multicolumn{4}{c}{\textbf{SDGL-NC}} \\
\cmidrule(lr){2-5}\cmidrule(lr){6-9}
\textbf{Dataset} & Number of layers $L$ & $d$ & Dropout $p$ & $\mathrm{lr}$ & Number of layers $L$ & $d$ & Dropout $p$ & $\mathrm{lr}$ \\
\midrule

\multicolumn{9}{l}{\textbf{Node classification (homophilic)}} \\
\midrule
Cora            & 15 & 128 & 0.39 & 0.0031 & 15 & 128 & 0.39 & 0.0094 \\
CiteSeer        & 10 & 512 & 0.40 & 0.0026 & 8 & 64 & 0.40 & 0.0020 \\
PubMed          & 10 & 256 & 0.36 & 0.0044 & 10 & 128 & 0.39 & 0.0031 \\
Coauthor-CS      & 8 & 128 & 0.52 & 0.0020 & 8 & 128 & 0.40 & 0.0020 \\
Computers       & 8 & 128 & 0.44 & 0.0059 & 8 & 128 & 0.39 & 0.0306 \\
Photo           & 8 & 128 & 0.52 & 0.0020 & 10 & 128 & 0.32 & 0.0020 \\

\midrule
\multicolumn{9}{l}{\textbf{Node classification (heterophilic)}} \\
\midrule
Actor           & 8 & 128 &  0.40 & 0.0061 & 15 & 128 & 0.25 & 0.0020 \\
Squirrel        & 10 & 128 & 0.40 & 0.0061 & 10 & 128 & 0.25 & 0.0020 \\
Chameleon       & 8 & 128 & 0.52 & 0.0306 & 8 & 128 & 0.39 & 0.0306 \\

\midrule
\multicolumn{9}{l}{\textbf{Graph classification (FakeNewsNet)}} \\
\midrule
PolitiFact      & 15 & 128 & 0.19 & 0.0031 & 10 & 128 & 0.09 & 0.0031 \\
GossipCop       & 15 & 128 & 0.01 & 0.0030 & 10 & 128 & 0.09 & 0.0031 \\

\bottomrule
\end{tabular}%
\vskip -0.1in
\end{table*}

\end{document}